\documentclass[twoside]{article}

\usepackage[accepted]{aistats2026}
\usepackage[round]{natbib}

\usepackage[utf8]{inputenc} 
\usepackage[T1]{fontenc}    
\usepackage{hyperref}       
\usepackage{url}            
\usepackage{booktabs}       
\usepackage{amsfonts}       
\usepackage{nicefrac}       
\usepackage{microtype}      
\usepackage{xcolor}         
\usepackage{tabularx}       
\newcommand{\indep}{\perp \!\!\! \perp}

\usepackage{graphicx}
\usepackage{amsmath,amssymb}
\usepackage{amsthm}
\theoremstyle{plain}
\newtheorem{thm}{Theorem}
\newtheorem*{thm*}{Theorem}
\newtheorem{lem}{Lemma}

\newtheorem{asmp}{Assumption}

\theoremstyle{definition}

\usepackage{etoc}
\etocdepthtag.toc{mtchapter}
\etocsettagdepth{mtchapter}{subsection}
\etocsettagdepth{mtappendix}{none}

\newcommand{\argmin}{\operatorname{argmin\,}}
\usepackage{mathtools}
\usepackage{hyperref}
\hypersetup{
			colorlinks,
			citecolor=blue,
			filecolor=black,
			linkcolor=red,
			urlcolor=black
		} 
\hypersetup{
	pdfborderstyle={/S/S/W 0.4} 
}
\usepackage{multirow}

\begin{document}
	
	%
	
	%
	
	\twocolumn[
	\aistatstitle{Mixture Proportion Estimation and Weakly-supervised Kernel Test for Conditional Independence}
	
	\aistatsauthor{ Yushi Hirose \And Akito Narahara \And Takafumi Kanamori}
	
	\aistatsaddress{ Institute of Science Tokyo \And  Institute of Science Tokyo \And Institute of Science Tokyo, RIKEN AIP} ]
	
	\begin{abstract}
		Mixture proportion estimation (MPE) aims to estimate class priors from unlabeled data. This task is a critical component in weakly supervised learning, such as PU learning, learning with label noise, and domain adaptation. Existing MPE methods rely on the \textit{irreducibility} assumption or its variant for identifiability. In this paper, we propose novel assumptions based on conditional independence (CI) given the class label, which ensure identifiability even when irreducibility does not hold. We develop method of moments estimators under these assumptions and analyze their asymptotic properties. Furthermore, we present weakly-supervised kernel tests to validate the CI assumptions, which are of independent interest in applications such as causal discovery and fairness evaluation. Empirically, we demonstrate the improved performance of our estimators compared with existing methods and that our tests successfully control both type I and type II errors.\label{key}
	\end{abstract}
	
	\section{INTRODUCTION}\label{introduction}
	
	Mixture Proportion Estimation (MPE) is the problem of estimating the mixture proportions of underlying class distributions in unlabeled data. This work addresses a generalized MPE setting (Unlabeled-Unlabeled or UU setting) where we are given samples from two distinct mixture distributions, $U=\theta P+(1-\theta )N$ and $U'=\theta' P+(1-\theta' )N$. Here, $P$ and $N$ represent the positive and negative class probability distributions, respectively, and the objective is to estimate the unknown class priors $(\theta,\theta')$. This formulation is a strict generalization of the standard MPE setting (Positive-Unlabeled or PU setting), which assumes access to $P$, corresponding to the case where $\theta=1$.
	
	MPE is a critical component in various machine learning tasks. For instance, weakly-supervised learning such as positive-unlabeled learning \citep{elkan08pu,du14pu,kiryo17pu}, unlabeled-unlabeled learning \citep{nan19uu, nan21biuu}, and learning from pairwise similarity \citep{bao18sim} require known mixture proportions to train classifiers. Other applications of MPE include learning with label noise \citep{natar13noise,scott13noise,liu16noise}, anomaly detection \citep{sand14ano}, and domain adaptation under open set label shift \citep{garg22open}.
	
	Without any assumptions on the distributions $P$ and $N$, $(\theta,\theta')$ are not identifiable. 	To address this, \citet{blanchard2010semi} introduced the irreducibility assumption, which posits that $N$ is irreducible with respect to $P$. Intuitively, this means that the negative distribution cannot be expressed as a mixture containing the positive distribution. In the UU setting, the converse: irreducibility of $P$ with respect to $N$ is also required for identifiability \citep{scott2015rate}.
	 To date, most of the existing MPE algorithms \citep{scott2015rate,scott13noise,jain2016nonparametric,ramaswamy2016mixture,ivanov2020dedpul,bekker2020learning} have been developed under the irreducibility assumption or stricter conditions, such as the anchor set assumption \citep{ramaswamy2016mixture}. 
		
	However, the irreducibility assumption can be violated in practical applications as discussed by \citet{zhu23beyond}. To the best of our knowledge, \citet{yao2022rethinking} and \citet{zhu23beyond} are the only existing works that have attempted MPE beyond irreducibility, while they still have limitations. The regrouping method of \citet{yao2022rethinking} can mitigate estimation bias, but is not statistically consistent. \citet{zhu23beyond} derived a more general condition than irreducibility that requires the essential supremum of $P(Y|X=x)$, which is usually unavailable in practice. 
	
	This work investigates alternative assumptions and presents new identifiability results for MPE. Our assumptions are based on the underlying data structure: conditional independence given class label. 	The CI assumption is widely adopted across various machine learning domains. For instance, in text classification and spam filtering \citep{jurafsky24speech}, features are often assumed to be independent given the label. Similarly, multi-view learning paradigms, including co-training \citep{blum98com} and unsupervised-learning  \citep{song14multi,stein16uns,anand14tensor}, frequently leverage conditional independence between feature sets. The applications include web page categorization \citep{blum98com}, text-image categorization \citep{giesen21mom} and biological data such as flow cytometry \citep{song14multi}. We consider two structural assumptions, conditional independence (CI) and multivariate conditional independence (MCI), on which we develop the method of moments estimators.
	
	We also establish kernel test methods to verify the CI assumptions from only observed data from $U$ and $U'$ (weakly-supervised setting). Our tests are developed from non-parametric kernel tests such as HSIC \citep{gretton07kernal} and Kernel-based Conditional Independence test \citep{zhang11kci}. We show the test consistency under mild conditions. In contrast, existing MPE assumptions, such as irreducibility, are generally difficult to verify, and no such studies currently exist. Our tests have potential applications not only in MPE, but also in fairness evaluation \citep{meh22fair} and causal discovery \citep{gordon23causal}.
	
	Our contributions are summarized as follows.
	\begin{itemize}
		\item We propose method of moments estimators for MPE, under class-specific CI  (independence of two features given the class label) and MCI (independence of two features given the class label and additional features) assumptions. We show the asymptotic normality of estimators.		
		\item We establish kernel tests for CI and MCI using unlabeled data from $U$ and $U'$, which is not possible with existing kernel tests.
		\item We investigate the testing methods under two settings: one where the true mixture proportions are known, and another where they are unknown. We derive the asymptotic distributions of the proposed test statistics under $H_0$ and propose gamma approximation methods. 
		\item In the test setting of unknown mixture proportions, we propose a post-hoc testing method, where we plug estimated mixture proportions into the test statistics and estimate the modified mean and variance for gamma approximation.
	\end{itemize}


	\section{PROBLEM SETTING}
	Let $X$ and $Y$ be feature and binary label random variables that take values in $\mathcal{X}$ and $\{-1,1\}$ respectively. In this paper, we assume two unlabeled datasets are given:
	$$\begin{aligned}
		\{x^{(i)}\}^n_{i=1}&\sim U=\theta P+(1-\theta)N\\
		\{x'^{(i)}\}^{n'}_{i=1}&\sim U'=\theta'P+(1-\theta')N
	\end{aligned}$$
	where $P:=P(X|Y=1)$ and $N:=P(X|Y=-1)$ are probability distributions for positive and negative data. $U$ and $U'$ are unlabeled data distributions with different mixture proportions $\theta,\theta'\in[0,1]$ ($\theta>\theta'$, without loss of generality). Denote $M=n+n'$ and assume $\lim_{M\rightarrow\infty} M/n\rightarrow\nu$ and $\lim_{M\rightarrow\infty}M/n'\rightarrow\nu'$. The objective is to estimate $\theta,\theta'$ from unlabeled datasets $\{x^{(i)}\}^{n}_{i=1},\{x'^{(i)}\}^{n'}_{i=1}$. Under these settings, the labeled distributions $P$ and $N$ can be expressed as mixtures of the unlabeled distributions, $U$ and $U'$: 
	\begin{align}
		P&=\alpha_+U+(1-\alpha_+)U'\label{simul1}\\
		N&=\alpha_-U+(1-\alpha_-)U' \label{simul2}	
	\end{align}
	where $\alpha_+:=\frac{1-\theta'}{\theta-\theta'}$ and $\alpha_-:=\frac{-\theta'}{\theta-\theta'}$.
	
	These equations demonstrate that the labeled distributions can be recovered from the unlabeled distributions if $(\theta,\theta')$ are known. This property is utilized in weakly-supervised learning \citep{nan19uu,du14pu} and learning with label noise \citep{natar13noise,scott2015rate}. We also leverage this property for our proposed MPE and CI tests. Specifically in our MPE, we estimate $\alpha_+$ and $\alpha_-$, which is equivalent to estimating $\theta,\theta'$, and is more convenient for theoretical analysis.
	
	\section{MIXTURE PROPORTION ESTIMATION WITH CONDITIONAL INDEPENDENCE}
	In this section, we present a novel approach to mixture proportion estimation by leveraging conditional independence (CI) between features.
	
	We assume a two-dimensional feature $X=(X_1,X_2) $ that takes values in $\mathcal{X}_1 \times \mathcal{X}_2$. We denote the positive and negative distributions as $P_{12} := P(X_1,X_2|Y=1)$ and $N_{12} := P(X_1,X_2|Y=-1)$, and unlabeled distributions $U_{12} := \theta P_{12} + (1-\theta)N_{12}$ and $U'_{12} := \theta' P_{12} + (1-\theta')N_{12}$. Similarly, $P_{\tau}, N_{\tau},  U_{\tau}, U'_{\tau}, \forall \tau\in\{1,2\} $ represent the corresponding marginal distributions.	We define mixture distributions $F^{\alpha}_{\tau}:= \alpha U_\tau + (1 - \alpha) U'_\tau$ and their empirical versions $\hat F^{\alpha}_{\tau}:= \alpha \hat U_\tau + (1 - \alpha) \hat U'_\tau, \forall\tau\in\{1,2,12\}$. Note that $F^{\alpha_+}_{\tau}=P_{\tau}$ and $F^{\alpha_-}_{\tau}=N_{\tau}$ from (\ref{simul1}) and (\ref{simul2}). 
	
	In the following procedure, we focus on estimating $\alpha^*\in\{\alpha_+,\alpha_-\}$ and define $\bar\alpha^*\in\{\alpha_+,\alpha_-\}\setminus\{\alpha^*\}$. Note that a similar assumption and procedure are required to estimate $\bar\alpha^*$ and subsequently recover $(\theta,\theta')$.  We assume the following class-specific CI assumption:
	
	\begin{asmp}\label{asmp:ci}
		$F^{\alpha^*}_{12}= F^{\alpha^*}_{1}F^{\alpha^*}_{2}$
	\end{asmp}

	If $\alpha^*=\alpha_+$ (resp. $\alpha_-$), the assumption corresponds to $P_{12}=P_1P_2$ (resp. $N_{12}=N_1N_2$). If we have multi-dimensional features (e.g., $\mathcal{X}=\mathbb{R}^d, d>2$), we select two features that satisfy Assumption \ref{asmp:ci} depending on whether the target is $\alpha_+$ or $\alpha_-$.
	
	Assumption \ref{asmp:ci} enables the identification of the mixture proportion $\alpha^*$. This is because $F^{\alpha}_{12}$ is a mixture distribution of $U$ and $U'$, and it generally does not exhibit conditional independence unless $\alpha=\alpha^*$ under Assumption \ref{asmp:ci} as shown in Lemma \ref{lem:cimoment}. To derive a moment condition and the MPE estimator, we define two vector-valued functions, $g_1: \mathcal{X}_1 \rightarrow \mathbb{R}^d$ and $g_2: \mathcal{X}_2 \rightarrow \mathbb{R}^d$, along with their dot product $g_{12}(X) := g_1(X_1) \cdot g_2(X_2)$. For notational simplicity, we often omit function arguments when it is clear, e.g., $g(x)$ as $g$.  Under Assumption \ref{asmp:ci}, we can establish the following moment condition.
	
	\begin{lem}\label{lem:cimoment} 
		Define a moment function 
		$$m_{CI}(\alpha):=E_{F^{\alpha}_{12}}[g_{12}]-E_{F^{\alpha}_{1}F^{\alpha}_{2}}[g_{12}]$$
		Under Assumption~\ref{asmp:ci}, $m_{CI}(\alpha^*)=0$. Moreover, the quadratic equation $m_{CI}(\alpha)$ has real solutions if $(E_{P_1}[g_1]-E_{N_1}[g_1])\cdot(E_{P_2}[g_2]-E_{N_2}[g_2])\neq 0.$ 
	\end{lem}
	
	Note that for an arbitrary $\alpha\in\mathbb{R}$, the distribution $F^\alpha_{\tau}$ is not necessarily a valid probability distribution, but rather a signed measure. Nevertheless, we define its expectation as an integral w.r.t. $F^\alpha_{\tau}$. Lemma \ref{lem:cimoment} implies that $\alpha^*$ can be estimated by finding the roots of the equation $m_{CI}(\alpha)=0$. Therefore, we define our empirical estimator as:
	$$\hat \alpha_{CI}:=\underset{\alpha\in I_{\alpha^*}}{\argmin} \hat m_{CI}^2(\alpha)$$
	where $\hat m_{CI}(\alpha):=E_{\hat F^{\alpha}_{12}}[g_{12}]-E_{\hat F^{\alpha}_{1}\hat F^{\alpha}_{2}}[g_{12}]$
	and $I_{\alpha^*}$ is a bounded, closed set containing $\alpha^*$.  
	
	The search space $I_{\alpha^*}$ should be chosen to ensure that $\alpha^*$ is the unique solution within this interval. If solving the equation yields two solutions within $I_{\alpha^*}$, a disambiguation step may be necessary. This can be achieved by performing a second estimation with a different feature map $g'_{12}$ and selecting the solution that yields a small $\hat m^2_{CI}(\alpha)$ for both $g_{12}$ and $g'_{12}$.
	
	$\hat\alpha_{CI}$ is a variant of method of moments estimator and we can derive its asymptotic normality as follows. 
	
	\begin{thm}[Asymptotic normality of $\hat\alpha_{CI}$] \label{thm:cimpe_an}~\\
		Assume $g_1$ and $g_2$ are bounded and continuous and $\alpha^*$ is the unique solution of $m_{CI}(\alpha)=0$ in $I_{\alpha^*}$. Then,$ $
		\begin{align*}
			&\sqrt{M}(\hat{\alpha}_{CI}-\alpha^*) \xrightarrow{d}\\ &\mathcal{N}\left(0,\frac{\nu{\alpha^*}^2V_{U_{12}}[\tilde g_{12}]+\nu' (1-\alpha^*)^2V_{U'_{12}}[\tilde g_{12}]}
			{(\theta-\theta')^2(E_{F^{\bar\alpha^*}_{12}}[\tilde g_{12}])^2}\right)	  
		\end{align*}
		as $M\rightarrow\infty$, where $\tilde g_{12}:= (g_1-E_{F^{\alpha^*}_1}[g_1])\cdot(g_2-E_{F^{\alpha^*}_2}[g_2])
		$
	\end{thm}
	
	This theorem indicates the conditions under which the asymptotic variance of the estimator is minimized. A small variance is achieved by maximizing the denominator and minimizing the numerator. The denominator is maximized when both $|\theta-\theta'|$ and $|E_{F^{\bar\alpha^*}_{12}}[\tilde g_{12}]|$ are large. $|E_{F^{\bar\alpha^*}_{12}}[\tilde g_{12}]|$ becomes large when the functions $g_1$ and $g_2$ fluctuate significantly on $F^{\bar\alpha^*}_{12}$ around their averages on $F^{\alpha^*}_{12}$. This implies the functions $g_1$ and $g_2$ should effectively capture the distributional difference between $P$ and $N$. To minimize the numerator, the variance of $\tilde g_{12}$ should be small.
	In our experiments, we adopt identity functions for $g_1$ and $g_2$ as a practical choice, which is shown to perform well in Section \ref{subsec:exp_mpe}.

	\section{MIXTURE PROPORTION ESTIMATION WITH MULTIVARIATE CONDITIONAL INDEPENDENCE}\label{sec:mpemci}
	In this section, we consider a more general assumption, multivariate conditional independence (MCI), for MPE. Let us assume a multivariate feature $X=(X_1,X_2,X_S)$ that takes values in $\mathcal{X}_1 \times \mathcal{X}_2 \times \mathcal{X}_S$. Similarly to the previous section, we denote the joint and marginal distributions of positive, negative, and unlabeled data as $P_{\tau}, N_{\tau},  U_{\tau}, U'_{\tau}$ with the appropriate subscripts $\tau$. We define the mixture distribution $F^\alpha_{\tau}:=\alpha U_{\tau}+(1-\alpha )U'_{\tau}$ and denote the conditional distribution by $F^\alpha_{\tau|S}:=F^\alpha_{\tau S}/F^\alpha_{S}$ with a slight abuse of notation.
	
	Suppose we aim to estimate $\alpha^*\in\{\alpha_+,\alpha_-\}$. We define the following class-specific MCI assumption:
	
	\begin{asmp}\label{asmp:mci}
		$F_{12S}^{\alpha^*}=F_{1|S}^{\alpha^*}F_{2|S}^{\alpha^*}F_{S}^{\alpha^*}$
	\end{asmp}  
	
	As in CI MPE, we do not need to use the same feature triplet $(X_1,X_2,X_S)$ to estimate both $\alpha_+$ and $\alpha_-$, if there are additional feature variables available. We can select a triplet that satisfies Assumption \ref{asmp:mci} depending on whether the target is $\alpha_+$ or $\alpha_-$.  For MCI MPE, we use scalar-valued functions  $g_1:\mathcal{X}_1\rightarrow\mathbb{R}$ and $g_2:\mathcal{X}_2\rightarrow\mathbb{R}$. By defining the conditional means $\mu_1^\alpha \left(x_S\right):=E_{F^\alpha _{1S}}[g_1(X_1)|X_S=x_S]$ and $\mu_2^\alpha \left(x_S\right):=E_{F^\alpha _{2S}}[g_2(X_2)|X_S=x_S]$, we can establish the follwoing lemma under Assumption \ref{asmp:mci}.
	\begin{lem}\label{lem:mcimoment} 
		Define
		\begin{align*}
			m_{MCI}&(\alpha):=\\
			&E_{F^{\alpha}_{12S}}\left[(g_1\left(X_1\right)-\mu^{\alpha}_1(X_S)) (g_2\left(X_2\right)-\mu^{\alpha}_2(X_S))\right]
		\end{align*}
		Under Assumption~\ref{asmp:mci}, $m_{MCI}(\alpha^*)=0$.
	\end{lem}
	Based on this moment condition, we define the empirical estimator of $\alpha^*$ as
	$$\hat \alpha_{MCI}:=\underset{\alpha\in I_{\alpha^*}}{\argmin} \hat m_{MCI}^2(\alpha)$$
	where $\hat m_{MCI}(\alpha):=$ $$E_{\hat F^{\alpha}_{12S}}\left[(g_1\left(X_1\right)-\mu^{\alpha}_1(X_S)) (g_2\left(X_2\right)-\mu^{\alpha}_2(X_S))\right]$$
	and $I_{\alpha^*}$ is a bounded and closed set that contains $\alpha^*$. Since we do not know the true conditional means $\mu^\alpha_1(x_S)$ and  $\mu^\alpha_2(x_S)$, we estimate them using kernel ridge regression (KRR) in a weakly-supervised manner from the two sets of unlabeled data. The empirical mean squared error for $\mu^\alpha_\tau, \forall \tau\in\{1,2\}$ is written as 
	\begin{align*}
		MSE_{\mu^\alpha_\tau}=&\frac{\alpha}{n}||\mathbf{g}_\tau(\mathbf{v}_{x_\tau})_{:n}-K_{:n,:}\mathbf{w}||^2\\
		&+\frac{1-\alpha}{n'}||\mathbf{g}_\tau(\mathbf{v}_{x_\tau})_{n:}-K_{n:,:}\mathbf{w}||^2+\lambda\mathbf{w}^TK\mathbf{w}
	\end{align*}  
	where $\mathbf{w}\in\mathbb{R}^M$ is the weight vector, $\lambda$ is a regularization parameter, and $\mathbf{v}_{x_\tau}:=(x_{\tau}^{(1)},...,x^{(n)}_{\tau},x'^{(1)}_{\tau},...,x'^{(n')}_{\tau})^T$ 
	is the feature vector of $x_\tau$ for the unlabeled data. $\mathbf{g}_\tau(\mathbf{v}_{x_\tau})$ is a vector representing the element-wise application of $g_\tau$ to $\mathbf{v}_{x_\tau}$. $K\in\mathbb{R}^{M\times M}$ is the Gram matrix of $x_S$ with a kernel $k(\cdot,\cdot)$, where $K_{ij}=k(\mathbf{v}_{x_S,i},\mathbf{v}_{x_S,j})$ . The terms $\mathbf{g}_\tau(\mathbf{v}_{x_\tau})_{:n}\in\mathbb{R}^n$, $\mathbf{g}_\tau(\mathbf{v}_{x_\tau})_{n:}\in\mathbb{R}^{n'}$, $K_{:n,:}\in\mathbb{R}^{n\times M}$ and $K_{n:,:}\in\mathbb{R}^{n'\times M}$ are subvectors of $\mathbf{g}_\tau(\mathbf{v}_{x_\tau})$ and submatrices of $K$. 
	
	Denoting the estimated parameter as $\mathbf{w_\alpha}=\argmin_\mathbf{w}MSE_{\mu^\alpha_\tau}$, the estimated conditional mean $\hat\mu^\alpha_\tau(\mathbf{v}_{x_S})=K\mathbf{w}_\alpha$ is used for MCI MPE. Note that $MSE_{\mu^\alpha_\tau}$ can be non-convex when $\alpha\notin [0,1]$. In such cases, we cannot obtain an explicit optimal solution as in standard KRR. In practice, we instead solve the first order condition to derive $\mathbf{w}_\alpha$. Given these procedures, $\argmin_{\alpha\in I_{\alpha^*}} \hat m_{MCI}^2(\alpha)$ can become a bilevel optimization. For the efficient computation, we can use iterative search methods, such as the golden-section search \citep{kiefer53gss} to find the optimal $\alpha$.
	
	We now present the asymptotic normality of $\hat \alpha_{MCI}$, similarly to CI MPE.
	
	\begin{thm}[Asymptotic normality of $\hat\alpha_{MCI}$] \label{thm:mcimpe_an}
		Assume $g_1$ and $g_2$ are bounded and continuous, and that $\mu_\tau^\alpha\left(x_S\right)$ is bounded and differentiable w.r.t. $\alpha\in I_{\alpha^*}$ and  $x_S\in\mathcal{X}_S$. Assume $\frac{\partial}{\partial \alpha}\mu^\alpha_\tau\left(x_S\right)$ is bounded. Suppose $\alpha^*$ is the unique minimizer of $m^2_{MCI}(\alpha)$ in $I_{\alpha^*}$, then 
		$$\begin{aligned}
			\sqrt{M}&(\hat\alpha_{MCI}-\alpha^*)\xrightarrow{d}\\ &\mathcal{N}\left(0,\frac{\nu{\alpha^*}^2V_{U_{12S}}[\tilde g_{12S}]+\nu' (1-\alpha^*)^2V_{U'_{12S}}[\tilde g_{12S}]}
			{(\theta-\theta')^2
				(E_{F^{\bar\alpha^*}_{12S}}[\tilde g_{12S}])^2}\right)
		\end{aligned}$$
		where $\tilde g_{12S}(x_1,x_2,x_S):= (g_1\left(x_1\right)-\mu^{\alpha^*}_1(x_S)) (g_2\left(x_2\right)-\mu^{\alpha^*}_2(x_S))$.
	\end{thm}
	
	This theorem suggests the conditions for small asymptotic variance, as discussed in Theorem $\ref{asmp:ci}$. We use identity functions for $g_1$ and $g_2$ as a practical choice in our experiments. 
	
	\section{CI AND MCI TEST UNDER WEAKLY-SUPERVISED SETTING}\label{sec:cimcitest}
	In this section, we introduce statistical testing methods to test the CI and MCI assumptions using only unlabeled data. These tests allow us to verify the applicability of our proposed MPE method. Beyond this primary objective, these tests have broader applications, including causal discovery and fairness evaluation. In causal discovery, conditional independence tests are required to infer the causal graph of the underlying data \citep{peter17ele}. In fairness evaluation, it is crucial to determine whether a classifier's output or representation $f(X)$ is independent of a protected variable $Z$ (e.g., race or sex) given $Y$.
	
	Our proposed tests can also be framed as a CI test with a single unobserved confounder in the context of recent work \citep{gordon23causal,bijan23causal,ming24causal}. However, existing methods have limitations; for instance, the methods of \citet{gordon23causal} and \citet{bijan23causal} are restricted to discrete variables. While the test by \citet{ming24causal} could be adopted by introducing an index variable to denote a sample's origin ($U$ or $U'$), its application relies on a specific condition of integral equation.	
	
	In the following subsections, we first present a testing method that assumes the true mixture proportions are known. Since the proportions are unknown in advance, this setting does not directly verify Assumptions \ref{asmp:ci} and \ref{asmp:mci} for MPE. Nevertheless, the method remains valuable for the other applications mentioned above. Subsequently, we introduce a testing method without the true mixture proportions.

	\subsection{Weakly-supervised Kernel CI(WsKCI) Test with True Mixture Proportions}\label{subsec:citwith}
	To verify the CI assumption, we test $H_0 : X_1 {\indep} X_2 \mid  Y=y$, against $H_1 : X_1 \not\!\perp\!\!\!\perp X_2 \mid Y=y$ using unlabeled data. In the first setting, we assume to know the true mixture proportion $\alpha^*=\alpha_+$ for $y=1$ or $\alpha^*=\alpha_-$ for $y=-1$.  Let $k_1$ and $k_2$ be positive-definite and characteristic kernels \citep{gretton15simpler} on $\mathcal{X}_1$ and $\mathcal{X}_2$ respectively and let $\varphi_1, \varphi_2$ and $\mathcal{H}_1, \mathcal{H}_2$ be corresponding feature mappings and RKHSs. Our test is based on Hilbert-Schmidt Independence Criterion (HSIC) \citep{gretton07kernal} for $F^{\alpha^*}_{12}$: 
	\begin{align*}
		\bigl\|E_{F^{\alpha^*}_{12}}&[\varphi_1(X_1)\otimes\varphi_2(X_2)]\\
		&-E_{F^{\alpha^*}_{1}F^{\alpha^*}_{2}}[\varphi_1(X_1)\otimes\varphi_2(X_2)]\bigr\|^2_{\mathcal{H}},
	\end{align*}
	which is the squared Hilbert-Schmidt norm of cross-covariance operator and equals zero under $H_0$. Here, $\otimes$ denotes the tensor product of kernels \citep{schrab25prac}. By replacing the population distributions in the above statistic with their empirical counterparts, we define the following test statistic 
	\begin{align*}
		T_{CI} :=&\Bigl\| E_{\hat{F}_{12}^{\alpha^*}}\left[ \varphi_1(X_1)\otimes\varphi_2(X_2) \right]\\ &- E_{\hat{F}_{1}^{\alpha^*}\hat{F}_{2}^{\alpha^*}}\left[\varphi_1(X_1)\otimes\varphi_2(X_2)\right] \Bigr\|_{\mathcal{H}}^2
	\end{align*} 
	 To implement the test, we require the null distribution of $T_{CI}$. We derive the following theorem using an approach similar to that of HSIC \citep{gretton07kernal}.
	
	\begin{thm}[Asymptotic distribution of $T_{CI}$]\label{thm:tci_asyd} Assume $k_1$ and $k_2$ are translation invariant
		$c_0$-kernels as defined in \citep{gretton15simpler}. Then,
		
		(i) Under $H_0$, we have 
		\begin{align*}
			MT_{CI} \overset{d}{\rightarrow} \sum_{i, j = 1}^\infty \lambda_{1, i} \lambda_{2, j} \xi_{i, j}^2.
		\end{align*}
		where $\lambda_{1, i}$ and $\lambda_{2, j}$ are the eigenvalues of the integral operators associated with $k_1$ and $k_2$, and $\xi_{i,j}$s follow a multivariate normal distribution with mean $\mathbf{0}$ and covariances defined in Appendix \ref{appsubsec:proof_ci_test}. 
		
		(ii) Under $H_1$, we have $MT_{CI} \overset{p}{\rightarrow} \infty$.
	\end{thm}
	The proof is provided in Appendix \ref{appsubsec:proof_ci_test}. The null distribution is obtained by considering the Mercer's expansions of $k_1$ and $k_2$, and applying the Central Limit Theorem to them; the result under $H_1$ is given by \citet{gretton15simpler}.
	
	Empirically, we approximate the null distribution with a gamma distribution, an approach also used for the HSIC test \citep{gretton07kernal}. The parameters for this gamma approximation are determined by estimating the mean and variance of $MT_{CI}$. The following theorem provides the asymptotic expressions for these moments.
	
	\begin{thm}[Asymptotic mean and variance of $MT_{CI}$]\label{thm:tci_mv_w/alpha} Define the centralized kernel $\tilde k_{12}$ associated with the feature map $\tilde{\varphi}_{12}(x):= (\varphi_1(x_1) - E_{F^{\alpha^*}_1}[\varphi_1(x_1)]) \otimes (\varphi_2(x_2) - E_{F^{\alpha^*}_2}[\varphi_2(x_2)])$.
		 Under $H_0$ and as $M\rightarrow\infty$, we have
		\begin{align*}
			E[MT_{CI}] \rightarrow& 
			\nu {\alpha^*}^2 E_{x_{i_1}, x_{i_2}}[\tilde{k}_{12}(x_{i_1}, x_{i_1}) - \tilde{k}_{12}(x_{i_1}, x_{i_2})] 
			\\ + \nu' (1 - &\alpha^*)^2E_{x'_{q_1}, x'_{q_2}}[\tilde{k}_{12}(x'_{q_1}, x'_{q_1}) - \tilde{k}_{12}(x'_{q_1}, x'_{q_2})]
			\\
			V[MT_{CI}]
			\rightarrow & 
			2\nu^2 \sigma_{CI,2,0}^2 + 2\nu'^2 \sigma_{CI,0,2}^2 + 4\nu\nu' \sigma_{CI,1,1}^2 ,
		\end{align*}
		where $x_{i_1}$,$x_{i_2}$ and $x'_{q_1}$,$x'_{q_2}$ are i.i.d samples from $U_{12}$ and $U'_{12}$, and 
		\begin{align*}
			\sigma_{CI,2,0}^2:=& E_{x_{i_1}, x_{i_2}}\left[\left(E_{x_{q_1},x_{ q_2}}\left[\left\langle\check{\varphi}_{i_1, q_1}, \check{\varphi}_{i_2, q_2}\right\rangle\right]\right)^2\right],\\
			\sigma_{CI,0,2}^2:=& E_{x_{q_1}, x_{q_2}}\left[\left(E_{x_{i_1}, x_{i_2}}\left[\left\langle\check{\varphi}_{i_1, q_1}, \check{\varphi}_{i_2, q_2}\right\rangle\right]\right)^2\right], \\
			\sigma_{CI,1,1}^2:=&E_{x_{i_1},x_{q_2}}\left[\left(E_{x_{i_2},x_{ q_1}}\left[\left\langle\check{\varphi}_{i_1, q_1}, \check{\varphi}_{i_2, q_2}\right\rangle\right]\right)^2\right] .
		\end{align*}
		Here, we define $\check{\varphi}_{i_1, q_1}:=\alpha^* \tilde{\varphi}_{12}\left(x^{\left(i_1\right)}\right)+ \left(1-\alpha^*\right) \allowbreak \tilde{\varphi}_{12}\left(x^{\prime\left(q_1\right)}\right)$.
	\end{thm}
	
	The proof is based on the theory of two-sample V-statistics, as $T_{CI}$ belongs to this class.	By replacing the population distributions in the asymptotic expressions with their empirical counterparts, we can estimate the mean and variance from unlabeled data. The p-value of $MT_{CI}$ is then computed using the approximated null distribution and compared against a predefined significance level.
	
	\subsection{Weakly-supervised Kernel MCI(WsKMCI) Test with True Mixture Proportions}\label{subsec:mcitwith}
	To test the MCI assumption with unlabeled data, we define the null hypothesis $H_0 : X_1 {\indep} X_2 \mid  X_S,Y=y$ against the alternative $H_1 : X_1 \not\!\perp\!\!\!\perp X_2 \mid X_S,Y=y$. We assume the true mixture proportion $\alpha^*=\alpha_+$ for $y=1$ or $\alpha^*=\alpha_-$ for $y=-1$ is given. Let $k_S$ be a positive-definite and characteristic kernel on $\mathcal{X}_S$ with a feature mapping $\varphi_S$ and RKHS $\mathcal{H}_S$. Our proposed test is based on the Kernel-based Conditional Independence (KCI) criterion \citep{zhang11kci, pogodin25prac} applied to $F^{\alpha^*}_{12S}$:
	\begin{align*}
		\Bigl\|E_{F_{12S}^{\alpha^*}}&[(\varphi_1(X_1) - \mu_{X_1 \mid X_S}(X_S)) \otimes \varphi_S(X_S) \otimes\\
		&(\varphi_2(X_2) - \mu_{X_2 \mid X_S}(X_S))]\Bigr\|^2_{\mathcal{H}}
	\end{align*}
	which is zero under $H_0$. Here, we define $\mu_{X_\tau \mid X_S}(X_S)$$:=E_{F_{\tau S}^{\alpha^*}}[\varphi_\tau(X_\tau)|X_S], \forall\tau\in\{1,2\}$. Analogously to the previous subsection, we define the empirical test statistic:
	\begin{align*}
		T_{MCI}:=\Bigl\|E_{\hat F_{12S}^{\alpha^*}}&[(\varphi_1(X_1) - \mu_{X_1 \mid X_S}(X_S)) \otimes \varphi_S(X_S) \otimes\\
		&(\varphi_2(X_2) - \mu_{X_2 \mid X_S}(X_S))]\Bigr\|^2_{\mathcal{H}}.
	\end{align*} 
	The testing procedure is identical to that of $T_{CI}$. We approximate the null distribution with gamma distribution by estimating its asymptotic mean and variance as in  \citet{zhang11kci}. The following theorem establishes the asymptotic null distribution and the consistency of the test under the assumptions used in \citet{fukumizu07ci}.
	
	\begin{thm}[Asymptotic distribution of $MT_{MCI}$]\label{thm:tmci_asyd} 
		Denote $\ddot{X}:=(X_1, X_S)$ and $k_{\ddot{\mathcal{X}}}:= k_1k_S$. Assume $\mathcal{H}_1 \subset L^2(F^{\alpha^*}_1), \mathcal{H}_{2} \subset L^2(F^{\alpha^*}_2)$, and $\mathcal{H}_{S} \subset L^2(F^{\alpha^*}_S)$ where $L^2(P)$ is the space of the square integrable
		functions with probability $P$. Further assume that $k_{\ddot{\mathcal{X}}} k_{2}$ is a characteristic kernel on $(\mathcal{X}_1 \times \mathcal{X}_S) \times \mathcal{X}_2$, and that $\mathcal{H}_{S}+\mathbb{R}$ (the direct sum of the two RKHSs) is dense in $L^2\left(F^{\alpha^*}_S\right)$. Then,
		
		(i) Under $H_0$, we have
		$$MT_{MCI}\xrightarrow{d}\sum^\infty_{i,j,q=1}\lambda_{1,i}\lambda_{2,j}\lambda_{S,q}\xi^2_{ijq}.$$
		where $\lambda_{1, i}$, $\lambda_{2, j}$ and $\lambda_{S, q}$ are the eigenvalues of the integral operators associated with $k_1$, $k_2$ and $k_S$, and $\xi_{ijq}$s follow a multivariate normal distribution with mean $\mathbf{0}$ and covariances defined in Appendix \ref{appsubsec:proof_mci_test}. 	
		
		(ii) Under $H_1$, we have $MT_{MCI} \overset{p}{\rightarrow} \infty$.
		
	\end{thm}
	
	In addition, the asymptotic mean and variance of $MT_{MCI}$ are given in the following theorem. These theorems are proved similarly to Theorems \ref{thm:tci_asyd} and \ref{thm:tci_mv_w/alpha}.
	
	\begin{thm}[Asymptotic mean and variance of $T_{MCI}$]\label{thm:tmci_mv_w/alpha}
		Define the centralized kernel $\tilde k_{12S}$ associated with the feature map $\tilde{\varphi}_{12S}(x) := (\varphi_1(x_1) - \mu_{X_1|X_S}(x_S)) \otimes\varphi_S(x_S)\otimes (\varphi_2(x_2) -  \mu_{X_2|X_S}(x_S))$. Under $H_0$, as $M\rightarrow\infty$, 
		\begin{align*}
			E[M&T_{MCI}]\rightarrow\nu{\alpha^*}^2E_{x_i,x_j}[\tilde k_{12S}(x_i,x_i)-\tilde k_{12S}(x_i,x_j)]\\
			&+\nu'(1-\alpha^*)^2E_{x'_q,x'_r}[\tilde k_{12S}(x'_q,x'_q)-\tilde k_{12S}(x'_q,x'_r)]
			\\
			V[M&T_{MCI}]\rightarrow 
			2{\nu}^2\sigma^2_{MCI,2,0}+2{\nu'}^2\sigma^2_{MCI,0,2}\\
			&+4\nu\nu'\sigma^2_{MCI,1,1}
		\end{align*}
		where $x_{i_1}$,$x_{i_2}$ and $x'_{q_1}$,$x'_{q_2}$ are i.i.d samples from $U_{12S}$ and $U'_{12S}$, and 
		\begin{align*}
			\sigma_{MCI,2,0}^2:=& E_{x_{i_1}, x_{i_2}}\left[\left(E_{x_{q_1}, x_{q_2}}\left[\left\langle\check{\varphi}_{i_1, q_1}, \check{\varphi}_{i_2, q_2}\right\rangle\right]\right)^2\right],\\
			\sigma_{MCI,0,2}^2:=& E_{x_{q_1}, x_{q_2}}\left[\left(E_{x_{i_1},x_{ i_2}}\left[\left\langle\check{\varphi}_{i_1, q_1}, \check{\varphi}_{i_2, q_2}\right\rangle\right]\right)^2\right], \\
			\sigma_{MCI,1,1}^2:=&E_{x_{i_1}, x_{q_2}}\left[\left(E_{x_{i_2},x_{ q_1}}\left[\left\langle\check{\varphi}_{i_1, q_1}, \check{\varphi}_{i_2, q_2}\right\rangle\right]\right)^2\right] .
		\end{align*}
		Here, we redefine $\check{\varphi}_{i_1, q_1}:=\alpha^* \tilde{\varphi}_{12S}\left(x^{\left(i_1\right)}\right)+ \left(1-\alpha^*\right) \allowbreak \tilde{\varphi}_{12S}\left(x^{\prime\left(q_1\right)}\right)$, analogously to the CI setting.
	\end{thm}
	
	As with $T_{CI}$, the mean and variance are estimated using their empirical counterparts. For the MCI test, however, we must also estimate the conditional kernel mean $\mu_{X_\tau \mid X_S}$. Following the procedure of \citet{zhang11kci}, we use the empirical kernel maps of $k_1$ and $k_2$ and then apply kernel ridge regression to estimate these conditional means. This estimation is performed in the weakly-supervised manner described in Section \ref{sec:mpemci}. Due to space constraints, full details are provided in Appendix \ref{appsubsec:comp_test_stat}.
	
	\subsection{CI and MCI Test without True Mixture Proportions}\label{subsec:testwo}
	The testing methods proposed in the previous subsections require known mixture proportions. Consequently, they cannot be used to verify Assumptions \ref{asmp:ci} and \ref{asmp:mci} to assess the MPE applicability, as these proportions are unknown in advance. Therefore, we propose a ``plug-in'' approach for the Cl and MCl tests without true mixture proportions. The validity of the plug-in test statistic is established in Lemma \ref{lem:asymp_exp} and Theorem \ref{thm:test_consistency_w/oalpha}. In this subsection, we use the following definitions for each test.
	\begin{align*}
		T_{CI,\alpha} :=&\Bigl\| E_{\hat{F}_{12}^{\alpha}}\left[ \varphi_1(X_1)\otimes\varphi_2(X_2) \right]-\\ &E_{\hat{F}_{1}^{\alpha}\hat{F}_{2}^{\alpha}}\left[\varphi_1(X_1)\otimes\varphi_2(X_2)\right] \Bigr\|_{\mathcal{H}}^2\\
		T_{MCI,\alpha}:=&\Bigl\|E_{\hat F_{12S}^{\alpha}}[(\varphi_1(X_1) - \mu_{X_1 \mid X_S}(X_S)) \otimes \varphi_S(X_S) \otimes\\
		&(\varphi_2(X_2) - \mu_{X_2 \mid X_S}(X_S))]\Bigr\|^2_{\mathcal{H}}
	\end{align*}
	where $\mu_{X_\tau \mid X_S}(X_S):=E_{F_{\tau S}^{\alpha^*}}[\varphi_\tau(X_\tau)|X_S], \forall\tau\in\{1,2\}$. Note that $T_{CI,\alpha^*}=T_{CI}$ and $T_{MCI,\alpha^*}=T_{MCI}$.
	
	Our proposed approach is as follows: we first estimate the mixture proportion by $\hat{\alpha}_{CI}$ (resp. $\hat{\alpha}_{MCI}$) and then use the plug-in test statistic $T_{C I,\hat\alpha_{CI}}$ (resp. $T_{MCI,\hat\alpha_{MCI}}$) \footnote{In practice, the true conditional kernel means required for the $T_{MCI,\hat\alpha_{MCI}}$ is estimated with $\hat F_{12S}^{\hat\alpha_{MCI}}$} instead of the original statistic $T_{CI,\alpha^*}$ (resp. $T_{MCI,\alpha^*}$). We use gamma approximation to derive the null distribution and conduct the statistical test, following a similar procedure to the ones from the previous subsections.
	
	Since the mean and variance of the plug-in statistic deviate from those of the original statistic, we derive them based on the following lemma.
	
	\begin{lem}\label{lem:asymp_exp} 
		Denote $T_{\alpha}:=T_{CI,\alpha}$ and $\hat\alpha:=\hat\alpha_{CI}$ for the CI test ($T_{\alpha}:=T_{MCI,\alpha}$ and $\hat\alpha:=\hat\alpha_{MCI}$ for the MCI test). Then, the following convergence holds for both tests. Under $H_0$, as $M\rightarrow\infty$,
		
		$MT_{\hat\alpha}-M\{T_{\alpha^*}+(\hat\alpha-\alpha^*)T'_{\alpha^*}+\frac{1}{2}(\hat\alpha-\alpha^*)^2T''_{\alpha^*}\}\xrightarrow{p}0$
		
		where $T'_{\alpha^*}:=\frac{d}{d\alpha}T_{\alpha}|_{\alpha=\alpha^*}$ and $T''_{\alpha^*}$ $:=\frac{d^2}{d\alpha^2}T_{\alpha}|_{\alpha=\alpha^*}$.
	\end{lem}
	
	This lemma is derived by the Taylor expansion of $T_{\hat\alpha}$ around $\alpha^*$. Considering the probabilistic limit, we approximate the mean and variance of $MT_{\hat\alpha}$ for each test as follows. The asymptotic equalities hold under mild conditions such as uniform integrability.
	\begin{align*}
		E[MT_{\hat\alpha}]&\simeq ME[T_{\alpha^*}+(\hat\alpha-\alpha^*)T'_{\alpha^*}+\frac{1}{2}(\hat\alpha-\alpha^*)^2T''_{\alpha^*}]\\
		V[MT_{\hat\alpha}]&\simeq M^2V[T_{\alpha^*}+(\hat\alpha-\alpha^*)T'_{\alpha^*}+\frac{1}{2}(\hat\alpha-\alpha^*)^2T''_{\alpha^*}]. 
	\end{align*}

	Each term on the r.h.s. can be estimated with the theory of U-statistics. However, as the derivation is rather complicated, we defer the full details to Appendix \ref{appsubsec:mean_var_wo_prop}.
	
	To ensure the test correctly rejects $H_0$ under $H_1$ (the test consistency), the following additional assumption is required, since $\hat{\alpha} \overset{p}{\rightarrow} \alpha^*$ is not guaranteed under $H_1$.
	
	\begin{asmp}\label{asmp:testconswomp} 
		(i) For the CI test, $\hat\alpha_{CI} \overset{p}{\rightarrow} \alpha_1$, such that  $F_{12}^{\alpha_1}$ is a probability distribution and $F_{12}^{\alpha_1}\neq F_{1}^{\alpha_1}F_{2}^{\alpha_1}$.
		
		(ii) For the MCI test, $\hat\alpha_{MCI} \overset{p}{\rightarrow} \alpha_1$, such that $F_{12S}^{\alpha_1}$ is a probability distribution and $F_{12S}^{\alpha_1}\neq F_{1|S}^{\alpha_1}F_{2|S}^{\alpha_1}F_{S}^{\alpha_1}$.
	\end{asmp}
	
	Assumption \ref{asmp:testconswomp} ensures that the limiting distributions $F_{12}^{\alpha_1}$ and $F_{12S}^{\alpha_1}$ do not exhibit spurious CI and MCI conditions, which could lead to accepting $H_0$ under $H_1$. This assumption is considered mild since $F^{\alpha_1}_{\tau}$ is a mixture distribution of $P_\tau$ and $N_\tau$, and the features in $X$ generally exhibit dependence. Under this assumption, we establish test consistency in the following theorem. The proof is provided in Appendix \ref{appsubsec:proof_test_wo_prop}.
	
	\begin{thm}\label{thm:test_consistency_w/oalpha}
		Let Assumption \ref{asmp:testconswomp} hold. Then, 
		
		(i) For the CI test, under the assumptions in Theorem \ref{thm:tci_asyd} and $H_1$, $MT_{CI,\hat\alpha}\xrightarrow{p}\infty$.
		
		(ii) For the MCI test, under the assumptions in Theorem \ref{thm:tmci_asyd} and $H_1$, $MT_{MCI,\hat\alpha}\xrightarrow{p}\infty$.
	\end{thm}
	
	\begin{table*}[h!]\centering
		\caption{Mean absolute error for $\hat\theta'$ in the CI MPE experiment. The lowest value in each column is highlighted in bold.\\}
		\label{tbl:CI_MPE}
		\newcolumntype{Y}{>{\centering\arraybackslash}X}
		\begin{tabularx}{0.9\textwidth}{lYYYY}
			\toprule
			& Gaussian & Shuttle & Wine  & Dry Bean \\
			\midrule
			DEDPUL & 0.043    & 0.117   & 0.077 & 0.074    \\
			KM2    & 0.027    & 0.152   & 0.129 & 0.029    \\
			EN     & 0.063    & 0.075   & 0.110 & 0.107    \\
			CI MPE & \textbf{0.013}    & \textbf{0.053}   & \textbf{0.031} & \textbf{0.025}    \\
			\bottomrule
		\end{tabularx}
	\end{table*}
	
	 	\begin{table*}[h!]
		\centering
		\caption{Number of detected CI pairs and mean absolute error of $\hat\theta'$}
		\label{tab:realexp}
		\begin{tabular}{llcc}\\
			\toprule
			Dataset & Negative Class &\# Detected CI pairs & MAE of $\hat\theta'$\\
			\midrule
			\multirow{2}{*}{Breast Cancer} & B & 88 & $0.0284 \pm 0.0248$ \\
			& M & 86 & $0.0498 \pm 0.0612$ \\
			\midrule
			\multirow{2}{*}{DryBean} & SEKER & 2 & $0.0255 \pm 0.0252$ \\
			& HOROZ & 4 & $0.0179 \pm 0.0075$ \\
			\bottomrule
		\end{tabular}
	\end{table*}
	
	\begin{table*}[h!]
		\caption{Mean and standard deviation of the absolute errors for $(\hat\theta,\hat\theta')$  in the MCI MPE experiment. The first row corresponds to the PU setting, in which only the results for $\theta'$ are shown.\\}
		\centering
		\label{tbl:MCI_MPE}
		\setlength{\tabcolsep}{2.5pt}
		\begin{tabular}{lccc}\toprule
			$(\theta,\theta')$ & $n=n'=100$ & $500$ & $1000$ \\ \midrule
			(1,0.2)     & ( -- , 0.044 $\pm$ 0.037) & ( -- , 0.019 $\pm$ 0.012)  & ( -- , 0.015 $\pm$ 0.010)  \\
			(0.8,0.2)      & (0.048 $\pm$ 0.036, 0.044 $\pm$ 0.030) & (0.016 $\pm$ 0.013, 0.020 $\pm$ 0.014)  & (0.015 $\pm$ 0.011, 0.014 $\pm$ 0.010)  \\
			(0.5,0.2)     & (0.077 $\pm$ 0.079, 0.056 $\pm$ 0.048)    & (0.031 $\pm$ 0.024, 0.025 $\pm$ 0.019)     & (0.025 $\pm$ 0.017, 0.020 $\pm$ 0.014)     \\ \bottomrule
		\end{tabular}
	\end{table*}
	
	\section{EXPERIMENTS}
	\subsection{Mixture Proportion Estimation}\label{subsec:exp_mpe}
	We implement our MPE methods with synthetic data and benchmark datasets taken from the UCI machine learning repository. The estimated $\alpha_{\pm}$ by our methods are converted to $(\hat\theta, \hat\theta')=(\frac{1-\hat\alpha_-}{\hat\alpha_+-\hat\alpha_-},\frac{-\hat\alpha_-}{\hat\alpha_+-\hat\alpha_-})$ for comparison with existing MPE methods.

	\textbf{CI MPE with synthetic data} To evaluate our method, we utilize Gaussian data and three datasets from the UCI machine learning repository. We compare our CI MPE method with three existing MPE algorithms: DEDPUL \citep{ivanov2020dedpul}, EN \citep{elkan08pu} and KM2 \citep{ramaswamy2016mixture} that are developed under the irreducibility assumption or its variant. Since these baseline methods are designed for positive and unlabeled data, we set $\theta=1$ and focused on estimating only $\theta'$. 
	
	Gaussian data for each class is generated from a two-dimensional Gaussian distribution: $X_1\sim\mathcal{N}(Y,1),X_2\sim\mathcal{N}(Y,1)$.
	For UCI datasets, we choose Shuttle, Wine and Dry Bean datasets. The specific classes assigned as positive and negative for each dataset are detailed in Appendix \ref{appsubsec:expcimpe_syn}. The primary goal of this experiment is to validate our method when the CI assumption (Assumption \ref{asmp:ci}) holds, while the irreducibility assumption is violated. To create this scenario, we modified the datasets by the following procedure. At first, $20\%$ of the original positive data is transferred into negative data to break irreducibility. Then we split the features into two sets of equal dimension and sample each set independently, given the class $Y$, with replacement. This manually creates CI datasets. For each dataset, we conducted $3\times 10$ experiments. We select $\theta'$ from $\{0.2,0.5,0.7\}$. For each $\theta'$, we performed 10 trials by randomly drawing $n=n'=2000$ samples for the positive and unlabeled sets. Foxr our CI MPE method, we use identity functions for $g_1$ and $g_2$, as this setup performed well in preliminary experiments.
	
	The results in Table \ref{tbl:CI_MPE} are the averages over the 10 trials for each $\theta'$. The table shows that our method consistently outperforms the other methods, which yield unstable estimations when the irreducibility is violated. 
	 
 	\textbf{CI MPE with real-world data} We also implement experiments on real-world datasets to demonstrate the existence of CI features and the applicability of our MPE method. 
 	
 	In the experiments, we first use the HSIC test \citep{gretton07kernal} on labeled data to detect feature pairs $(X_1,X_2)$ that are conditionally independent given the negative class, i.e., $X_1\indep X_2|Y=-1$. Then, we conduct MPE experiments with our CI MPE method on the detected feature pairs. We use two datasets from the UCI repository: the Breast Cancer Wisconsin and Dry Bean datasets. For each dataset, we choose positive and negative classes and implement the experiments, switching the classes.
 	The detailed procedure is provided in Appendix \ref{appsubsec:expcimpe_real}.
 	 	
 	For the MPE task, we set $n=n^{\prime}=2000$ and used a Positive-Unlabeled (PU) setting with class $\operatorname{priors}\left(\theta, \theta^{\prime}\right)=(1,0.5)$.
 	The number of detected CI pairs and the resulting Mean Absolute Error (MAE) of all pairs and trials in each dataset are shown in Table \ref{tab:realexp}.
 	The results confirm the presence of CI pairs in real-world data and show that our MPE method is applicable in these scenarios.

	\textbf{MCI MPE with synthetic data} In this experiment, we use Gaussian data that satisfies the MCI assumption (Assumption \ref{asmp:mci}). We generate three-dimensional data as follows: $X_S\sim\mathcal{N}(0.5,1), X_1\sim\mathcal{N}(Y,1)+X_S, X_2\sim\mathcal{N}(Y,1)+X_S$. We use identity functions for $g_1$ and $g_2$ as in CI MPE, and employ Gaussian kernels for KRR. The Golden section search method \citep{kiefer53gss} is used to optimize $\hat m_{MCI}^2(\alpha)$. We performed 100 trials for each pair $ (\theta,\theta')\in\{(1,0.2),(0.8,0.2),(0.5,0.2)\}$.  Further details, including the regularization parameter for KRR and the search range $I_{\alpha^*}$, are available in Appendix \ref{appsubsec:expmcimpe_syn}.
	The averaged results, presented in Table \ref{tbl:MCI_MPE}, suggest that our method successfully estimates $\theta$ and $\theta'$. As shown in Theorem \ref{thm:mcimpe_an}, the estimation errors tend to decrease as $(\theta-\theta')^2$ increases.
	
	We also implement MCI MPE with real-world data. However, due to space constraints, these results are deferred to Appendix \ref{appsubsec:expmcimpe_real}.

	\subsection{Weakly-supervised Kernel CI and MCI Tests}\label{subsec:exptest}
	We evaluate the performance of the proposed kernel CI and MCI tests with the following class-conditional Gaussian data for each test:
%
\\CI: $\begin{pmatrix}
	X_1 \\
	X_2
\end{pmatrix} 
\sim \mathcal{N}\left(\begin{pmatrix}
	Y \\ 
	Y
\end{pmatrix}, \Sigma_Y \right)$ 

MCI: $\begin{pmatrix}
	X_1 \\ 
	X_2
\end{pmatrix} 
\sim \mathcal{N}\left(\begin{pmatrix}
	Y \\ 
	Y
\end{pmatrix}, \Sigma_Y \right) + 
\begin{pmatrix}
	X_S \\
	X_S
\end{pmatrix}$ and $X_S\sim\mathcal{N}(0.5,1)$. 

Here, the covariance matrix is defined as $\Sigma_1 = \begin{pmatrix} 1 & \sigma_{12} \\ \sigma_{12} & 1 \end{pmatrix}$ for the positive class, and the identity matrix $\Sigma_{-1} = I$ for the negative class. In these experiments, we specifically test the CI and MCI assumptions for the positive distribution. 
	
To generate data under both the null and alternative hypotheses, we varied the covariance $\sigma_{12}$. The null hypothesis $\left(H_0\right)$ corresponds to $\sigma_{12}=0$, while the alternative $\left(H_1\right)$ corresponds to $\sigma_{12} \neq 0$. A larger value of $\sigma_{12}$ represents a greater deviation from the null hypothesis. For all experiments, we set the true mixture proportions to $\left(\theta, \theta^{\prime}\right)=(0.8,0.2)$ and use a Gaussian kernel for our tests. We assess the  tests' ability to control the Type I and Type II error rates by performing 1000 repetitions for each experimental setting, varying the sample size. The significance level is set to 0.05. Further details such as hyperparameters are provided in Appendix \ref{appsubsec:exptest_hypara}.
	
	\begin{table}[t!]
		\centering
		\caption{Rejection rates for the kernel CI test (top) and MCI test (bottom) with true mixture proportions. They should be close to the significance level $0.05$ under $H_0$ ($\sigma_{12}$ = 0).\\}
		\label{tbl:test_with}
		\begin{minipage}{0.48\textwidth}
			\centering
			\begin{tabular}{lrrr}\toprule
				$\sigma_{12}$ & $n=n'=500$ & $1000$ & $2000$ \\ \midrule
				0        & 0.051 & 0.055  & 0.052  \\
				0.2      & 0.399 & 0.748  & 0.996  \\
				0.5      & 1     & 1      & 1      \\ \bottomrule\\
			\end{tabular}
		\end{minipage}\hspace{-2.5em}
		\begin{minipage}{0.48\textwidth}
			\centering
			\begin{tabular}{lrrr}\toprule
				$\sigma_{12}$ & $n=n'=500$ & $1000$ & $2000$ \\ \midrule
				0             & 0.062 & 0.042   & 0.035  \\
				0.2           & 0.307 & 0.605   & 0.910  \\
				0.5           & 0.988 & 0.999   & 1      \\ \bottomrule
			\end{tabular}
		\end{minipage}
	\end{table}
	
	\begin{table}[t!]
		\centering
		\caption{Rejection rates for the kernel CI test (top) and MCI test (bottom) without true mixture proportions. They should be close to the significance level $0.05$ under $H_0$ ($\sigma_{12}$ = 0).\\}
		\label{tbl:test_without}
		\begin{minipage}{0.48\textwidth}
			\centering
			\begin{tabular}{lrrr}\toprule
				$\sigma_{12}$ & $n=n'=500$ & $1000$ & $2000$ \\ \midrule
				0             & 0.041 & 0.038   & 0.042  \\
				0.2           & 0.573 & 0.915   & 0.994  \\
				0.5           & 1     & 1       & 1      \\ \bottomrule\\
			\end{tabular}
		\end{minipage}\hspace{-2.5em}
		\begin{minipage}{0.48\textwidth}
			\centering
			\begin{tabular}{lrrr}\toprule
				$\sigma_{12}$ & $n=n'=1000$ & $2000$ & $3000$ \\ \midrule
				0             &  0.040 &0.050    &0.055   \\
				0.2           & 0.207 &0.484    &0.743   \\
				0.5           &    0.726  & 0.92       &0.951      \\ \bottomrule
			\end{tabular}
		\end{minipage}
	\end{table}
	
	\textbf{Tests with true mixture proportions } In this setting, the methods from Sections \ref{subsec:citwith} and 
	\ref{subsec:mcitwith} are evaluated. As shown in Table \ref{tbl:test_with}, both methods effectively control the Type I error around the target significance level of 0.05. While statistical power is limited for smaller sample sizes ($n=500$) when  $\sigma_{12}=0.2$, it improves significantly for larger $n$.
	
	\textbf{Tests without true mixture proportions }
	We also evaluate the tests proposed in Section \ref{subsec:testwo}, with results presented in Table \ref{tbl:test_without}.  This setting is more challenging as the true mixture proportions are unknown. Consequently, the Type II error rates are higher compared to the previous experiment, particularly for the MCI test. As we detail in Appendix \ref{appsubsec:exp_low_power}, the low power in the MCI test stems from the null and alternative distributions not being well-separated in small samples. This limitation can be mitigated by increasing the sample size at the expense of computational cost. From the perspective of our downstream MCI MPE application, this lower power is not highly problematic. Even when the test fails to reject $H_0$ for $\sigma_{12}=0.2$, the resulting relative bias of the estimator $\hat{\alpha}_{M C I}$ is small (approximately 10\%), as detailed in Appendix \ref{appsubsec:exp_bias}.
	
	\section{CONCLUSIONS AND DISCUSSIONS}
	This work introduces novel identifiability conditions for mixture proportion estimation (MPE): the class-specific Conditional Independence (CI) and Multivariate Conditional Independence (MCI) assumptions.  Based on these conditions, we propose method of moments estimators and establish their theoretical properties.
	
	Another contribution is the development of weakly-supervised kernel-based statistical tests (WsKCI and WsKMCI) that validate these CI and MCI assumptions using only unlabeled data. These tests have potential applications such as causal discovery and fairness assessment. We empirically demonstrate the effectiveness of the proposed MPE estimators and statistical tests. 
	
	Several directions for future research arise from this study. First, the performance of our MPE methods depends on the choice of functions $g_1$ and $g_2$. Investigating how to choose $g_1$ and $g_2$ for the smaller estimation variance is an important direction. Second, conducting our statistical tests on high-dimensional data is challenging, as the number of candidate CI and MCI pairs can become large, which causes a multiple testing problem. Developing a method to efficiently find pairs that satisfy the MPE assumptions in high-dimensional data is left for future work. A third direction is to apply our test statistic as a regularizer for fair and domain-invariant representation learning, similarly to  \citet{pogodin23eff}.

	\section*{Acknowledgement}
	TK was partially supported by JSPS KAKENHI Grant Numbers 20H00576, 23H03460, and 24K14849.
	
	\bibliographystyle{apalike}
	\bibliography{ref}
	
	\section*{Checklist}
	
	\begin{enumerate}
		
		\item For all models and algorithms presented, check if you include:
		\begin{enumerate}
			\item A clear description of the mathematical setting, assumptions, algorithm, and/or model. [Yes] 
			\item An analysis of the properties and complexity (time, space, sample size) of any algorithm. [Yes]
			
			\item (Optional) Anonymized source code, with specification of all dependencies, including external libraries. [Yes]
			
		\end{enumerate}
		
		\item For any theoretical claim, check if you include:
		\begin{enumerate}
			\item Statements of the full set of assumptions of all theoretical results. [Yes]

			\item Complete proofs of all theoretical results. [Yes]

			\item Clear explanations of any assumptions. [Yes]

		\end{enumerate}
		
		\item For all figures and tables that present empirical results, check if you include:
		\begin{enumerate}
			\item The code, data, and instructions needed to reproduce the main experimental results (either in the supplemental material or as a URL). [Yes]

			\item All the training details (e.g., data splits, hyperparameters, how they were chosen). [Yes]
			
			\item A clear definition of the specific measure or statistics and error bars (e.g., with respect to the random seed after running experiments multiple times). [Yes]

			\item A description of the computing infrastructure used. (e.g., type of GPUs, internal cluster, or cloud provider). [Yes]

		\end{enumerate}
		
		\item If you are using existing assets (e.g., code, data, models) or curating/releasing new assets, check if you include:
		\begin{enumerate}
			\item Citations of the creator If your work uses existing assets. [Yes]

			\item The license information of the assets, if applicable. [Yes]

			\item New assets either in the supplemental material or as a URL, if applicable. [Not Applicable]

			\item Information about consent from data providers/curators. [Not Applicable]

			\item Discussion of sensible content if applicable, e.g., personally identifiable information or offensive content. [Not Applicable]

		\end{enumerate}
		
		\item If you used crowdsourcing or conducted research with human subjects, check if you include:
		\begin{enumerate}
			\item The full text of instructions given to participants and screenshots. [Not Applicable]

			\item Descriptions of potential participant risks, with links to Institutional Review Board (IRB) approvals if applicable. [Not Applicable]

			\item The estimated hourly wage paid to participants and the total amount spent on participant compensation. [Not Applicable]
			
		\end{enumerate}	
	\end{enumerate}

	\clearpage
	\appendix
	\thispagestyle{empty}
	
	\newpage
	
%
%




%

%

\onecolumn
\aistatstitle{Appendix}
\etocdepthtag.toc{mtappendix}
\etocsettagdepth{mtchapter}{none}
\etocsettagdepth{mtappendix}{subsection}
\tableofcontents

\section{MIXTURE PROPORTION ESTIMATION WITH CI AND MCI}
\subsection{Proofs for the MPE with CI}
\begin{proof}[Proof of Lemma \ref{lem:cimoment}]
	$m_{CI}(\alpha)=0$ is a quadratic equation such that 
	$$m_{CI}(\alpha)=a\alpha^2+b\alpha+c$$
	where 
	\begin{align*}
		a&=E_{U_1U'_2}[g_{12}]+E_{U'_1U_2}[g_{12}]-E_{U_1U_2}[g_{12}]-E_{U'_1U'_2}[g_{12}],\\
		b&=E_{U_{12}}[g_{12}]-E_{U'_{12}}[g_{12}]+2E_{U'_1U'_2}[g_{12}]-E_{U_1U'_2}[g_{12}]-E_{U'_1U_2}[g_{12}],\\
		c&=E_{U'_{12}}[g_{12}]-E_{U'_1U'_2}[g_{12}].
	\end{align*}
	
	The distributions for the coefficient $a$ are simplified as follows:
	$$U_1U'_2+U_1U'_2-U_1U_2-U'_1U'_2=-(U_1-U'_1)(U_2-U'_2)=-(\theta-\theta')^2(P_1-N_1)(P_2-N_2)$$
	
	Therefore, 
	$$a=-(\theta-\theta')^2(E_{P_1}[g_1]-E_{N_1}[g_1])\cdot(E_{P_2}[g_2]-E_{N_2}[g_2]).$$
	
	Considering $\alpha^*$ is one solution of $m_{CI}(\alpha)=0$, if $(E_{P_1}[g_1]-E_{N_1}[g_1])\cdot(E_{P_2}[g_2]-E_{N_2}[g_2])\neq0$, $a\neq 0$ and there exist real solutions for $m_{CI}(\alpha)=0$.
\end{proof}

\begin{proof}[Proof of Theorem \ref{thm:cimpe_an}]
	The proof first requires establishing the consistency of the estimator $\hat\alpha_{CI}$.
	\begin{lem}[Consistency of $\hat\alpha_{CI}$]\label{lem:cimpe_cons}
		Assume $\alpha^*$ is the unique solution of $m_{C I}(\alpha)=0$ in $I_{\alpha^*}$. Then, $\hat\alpha_{CI}$ is a consistent estimator of $\alpha^*$.
	\end{lem}
	
	\begin{proof}[Proof of Lemma \ref{lem:cimpe_cons}]
		For a fixed $\alpha, \hat{m}_{C I}^2(\alpha)$ can be viewed as a two-sample V-statistic \citep{yan24data,Vaart98asymp}. Following a similar procedure to the proof of Theorem 3 in \citet{yan24data}, we can establish the uniform convergence: $\sup _{\alpha \in I_{\alpha^*}}\left|\hat{m}_{C I}^2(\alpha)-m_{C I}^2(\alpha)\right| \xrightarrow{p} 0$. Then, we can prove the consistency, $\hat{\alpha}_{C I} \xrightarrow{p} \alpha^*$, similarly to Theorem 5.7 of \citet{Vaart98asymp}.
	\end{proof}

	For simplicity, we denote $m_{CI}(\alpha)$, $\hat m_{CI}(\alpha)$, and $\hat\alpha_{CI}$ as $m(\alpha)$, $\hat m(\alpha)$ and $\hat\alpha$ respectively in this proof.
	Since $\hat{\alpha}_{CI}$ minimizes $\hat m^2(\alpha)$,
	by the first-order condition, we have 
	\begin{equation}\label{eq:firstOrder}
		\frac{\partial\hat{m}^2}{\partial \alpha}(\hat{\alpha}) = 2 \hat{m}(\hat{\alpha}) \frac{\partial \hat{m}}{\partial \alpha}(\hat{\alpha}) = 0, 
	\end{equation}
	where we denote $\frac{\partial\hat{m}^2}{\partial \alpha}(\hat{\alpha}):=\frac{\partial\hat{m}^2}{\partial \alpha}(\alpha)|_{\alpha=\hat\alpha}$. 
	
	On the other hand, by the mean value theorem, we get
	\begin{equation}\label{eq:meanValue}
		\hat{m}(\hat{\alpha}) - \hat{m}(\alpha^*) 
		= 
		\frac{\partial \hat{m}}{\partial \alpha}(\tilde{\alpha}) (\hat{\alpha} - \alpha^*)
	\end{equation}
	for some $\tilde{\alpha}$ between $\alpha^*$ and $\hat{\alpha}$.
	Multiplying both sides of (\ref{eq:meanValue}) by $\frac{\partial \hat{m}}{\partial \alpha}(\hat{\alpha})$ and applying (\ref{eq:firstOrder}) yields
	$$
	-\frac{\partial \hat{m}}{\partial \alpha}(\hat{\alpha}) \hat{m}(\alpha^*)=
	\frac{\partial \hat{m}}{\partial \alpha}(\hat{\alpha}) \frac{\partial \hat{m}}{\partial \alpha}(\tilde{\alpha}) (\hat{\alpha} - \alpha^*),
	$$
	and thus
	$$
	\hat{\alpha} - \alpha^*
	=-\frac{\frac{\partial \hat{m}}{\partial \alpha}(\hat{\alpha}) \hat{m}(\alpha^*) }{\frac{\partial \hat{m}}{\partial \alpha}(\hat{\alpha}) \frac{\partial \hat{m}}{\partial \alpha}(\tilde{\alpha})}.
	$$

	Here, $\frac{\partial \hat{m}}{\partial \alpha}(\hat{\alpha})$ and $\frac{\partial \hat{m}}{\partial \alpha}(\tilde{\alpha})$ converge to $\frac{\partial m}{\partial \alpha} (\alpha^*)$ in probability because $\hat{\alpha}, \tilde{\alpha} \overset{p}{\rightarrow} \alpha^*$ holds by consistency (Lemma \ref{lem:cimpe_cons}). Therefore, assuming $\sqrt{M} \hat{m}(\alpha^*) \overset{d}{\rightarrow} \mathcal{D}$ for some distribution $\mathcal{D}$, we have
	\begin{equation*}
		\sqrt{M} (\hat{\alpha} - \alpha^*) \overset{d}{\rightarrow} - \frac{1}{d_0} \mathcal{D},
	\end{equation*}
	where $d_0 := \frac{\partial m}{\partial \alpha} (\alpha^*)$. 
	
	Next, we identify the limiting distribution $\mathcal{D}$. We can write  $\hat{m}(\alpha^*)$ as
	\begin{equation*}
		\hat{m}(\alpha^*) 
		=
		E_{\hat{F}^{\alpha^*}_{12}}[(g_{1} - E_{\hat{F}^{\alpha^*}_1}[g_{1}]) \cdot (g_{2} - E_{\hat{F}^{\alpha^*}_2}[{g_{2}}])].
	\end{equation*}
	
	We introduce $\hat{m}_c(\alpha^*)$ with the true centralization, defined as
	\begin{equation*}
		\hat{m}_c(\alpha^*) 
		:=
		E_{\hat F^{\alpha^*}_{12}}[(g_{1} - E_{F^{\alpha^*}_1}[g_{1}]) \cdot (g_{2} - E_{F^{\alpha^*}_2}[{g_{2}}])]. 
	\end{equation*}
	These two quantities satisfy $\sqrt{M}(\hat{m}(\alpha^*) - \hat{m}_c(\alpha^*)) = o_p(1)$ because
	\begin{align*}
		\sqrt{M}(\hat{m}(\alpha^*) - \hat{m}_c(\alpha^*))
		&=
		\sqrt{M} (E_{\hat F_1^{\alpha^*}}[{g_{1}}] - E_{F^{\alpha^*}_1}[{g_{1}}]) \cdot (E_{\hat F_2^{\alpha^*}}[{g_{2}}] - E_{F^{\alpha^*}_2}[{g_{2}}])
		\overset{p}{\rightarrow}
		0,
	\end{align*}
	
	where the first term,$\sqrt{M} (E_{\hat F^{\alpha^*}_1}[{g_{1}}] -E_{F^{\alpha^*}_1}[{g_{1}}])$ converges in distribution to a normal distribution by the Central Limit Theorem (CLT), while the second term $E_{\hat F_2^{\alpha^*}}[{g_{2}}] -E_{F^{\alpha^*}_2}[{g_{2}}]$ converges in probability to $\mathbf{0}$ by the Law of Large number (LLN).
	Therefore, $\sqrt{M}\hat{m}(\alpha^*)$ and $\sqrt{M}\hat{m}_c(\alpha^*)$ converge in distribution to the same distribution $\mathcal{D}$ (if they converge).
	Let  $\tilde{g}_{12} := (g_{1} - E_{F^{\alpha^*}_1}[{g_{1}}]) \cdot (g_{2} - E_{F^{\alpha^*}_2}[{g_{2}}])$ be the centralized function. Then we have 
	\begin{align*}
		\sqrt{M}\hat{m}_c(\alpha^*) 
		=&
		\sqrt{M}E_{\hat{F}^{\alpha^*}_{12}}[{\tilde{g}_{12}}]
		\\ =&
		\alpha^*\sqrt{M}E_{\hat{U}_{12}}[{\tilde{g}_{12}}]
		+
		(1 - \alpha^*)\sqrt{M}E_{\hat{U'}_{12}}[{\tilde{g}_{12}}]
		\\ =&
		\alpha^*\sqrt{M}E_{\hat{U}_{12}}[{\tilde{g}_{12}}]
		+
		(1 - \alpha^*)\sqrt{M}E_{\hat{U'}_{12}}[{\tilde{g}_{12}}]
		- 
		(\alpha^*\sqrt{M}E_{U_{12}}[{\tilde{g}_{12}}]
		+
		(1 - \alpha^*)\sqrt{M}E_{U'_{12}}[{\tilde{g}_{12}}]
		)
		\\ =&
		\alpha^*\sqrt{M/n} \sqrt{n}(
		E_{\hat U_{12}}[{\tilde{g}_{12}}]
		-
		E_{U_{12}}[{\tilde{g}_{12}}])
		+
		(1 - \alpha^*)\sqrt{M/n'} \sqrt{n'}(
		E_{\hat U'_{12}}[{\tilde{g}_{12}}]
		-
		E_{U'_{12}}[{\tilde{g}_{12}}])
		\\ \overset{d}{\rightarrow}&
		\mathcal{N}(0, \nu {\alpha^*}^2 V_{U_{12}}[\tilde{g}_{12}] + \nu' (1 - {\alpha^*})^2 V_{U'_{12}}[\tilde{g}_{12}]), 
	\end{align*}
	where the third equality holds under Assumption \ref{asmp:ci} and we used the CLT in the last convergence.
	
	Combining these results, the asymptotic distribution of $\hat\alpha$ is 
	$$\sqrt{M}(\hat{\alpha}-\alpha^*) \xrightarrow{d} \mathcal{N}\left(0,\frac{\nu {\alpha^*}^2V_{U_{12}}[\tilde g_{12}]+\nu' (1-\alpha^*)^2V_{U'_{12}}[\tilde g_{12}]}
	{d^2_0}\right). $$
	
	Finally, we analyze the derivative term $d_0$. For $\alpha^*=\alpha_+$, we have
	\begin{align*}
		d_0&=\frac{\partial m}{\partial \alpha}(\alpha^*)=E_{U_{12}}[g_{12}]-E_{U'_{12}}[g_{12}]-E_{F^{\alpha^*}_{1}(U_2-U'_2)}[g_{12}]-E_{(U_1-U'_1)F^{\alpha^*}_2}[g_{12}]\\
		&=(\theta-\theta')(E_{F^{\alpha^*}_{12}}[g_{12}]-E_{ F^{\bar\alpha^*}_{12}}[g_{12}])-(\theta-\theta')(E_{F^{\alpha^*}_{1}F^{\alpha^*}_{2}}[g_{12}]-E_{F^{\alpha^*}_1 F^{\bar\alpha^*}_{2}}[g_{12}])\\
		&\quad-(\theta-\theta')(E_{F^{\alpha^*}_{1}F^{\alpha^*}_{2}}[g_{12}]-E_{F^{\bar\alpha^*}_1 F^{\alpha^*}_{2}}[g_{12}])\\
		&=-(\theta-\theta')E_{F^{\bar\alpha^*}_{12}}[\tilde g_{12}].
	\end{align*}
	For $\alpha^*=\alpha_-$, we have $d_0=(\theta-\theta')E_{F^{\bar\alpha^*}_{12}}[\tilde g_{12}]$.
	Then the desired asymptotic distribution is derived.
\end{proof}

%
%

\subsection{Proofs for the MPE with MCI}
\begin{proof}[Proof of Theorem \ref{thm:mcimpe_an}]
	For simplicity, we denote $m_{MCI}(\alpha)$, $\hat m_{MCI}(\alpha)$, and $\hat\alpha_{MCI}$ as $m(\alpha)$, $\hat m(\alpha)$ and $\hat\alpha$ respectively in this proof. For any fixed $\alpha$, $\hat m^2(\alpha)$ can be viewed as a two-sample V-statistic. We can show the uniform convergence $\sup _{\alpha \in I_{\alpha^*}}\left|\hat m^2(\alpha)-m^2(\alpha)\right| \xrightarrow{p} 0$ and prove the consistency of $\hat\alpha_{MCI}$, $\hat\alpha\xrightarrow{p}\alpha^*$ by an argument analogous to the proof of Lemma \ref{lem:cimpe_cons}.
	
	Furthermore, following the same procedure as in the proof of Theorem \ref{thm:cimpe_an}, if we assume $\sqrt{M} \hat{m}(\alpha^*) \overset{d}{\rightarrow} \mathcal{D}$ for some distribution $\mathcal{D}$, it follows that
	\begin{equation*}
		\sqrt{M} (\hat{\alpha} - \alpha^*) \overset{d}{\rightarrow} - \frac{1}{d_0}\mathcal{D}
	\end{equation*}
	where $d_0 := \frac{\partial m}{\partial \alpha} (\alpha^*)$. By the CLT, we also have
	\begin{align*}
		\sqrt{M}\hat{m}(\alpha^*)&=
		\alpha^*\sqrt{M}E_{\hat{U}_{12S}}[{\tilde{g}_{12S}}]
		+ 
		(1 - \alpha^*)\sqrt{M}E_{\hat{U'}_{12S}}[{\tilde{g}_{12S}}]
		\\ &=
		\alpha^*\sqrt{M}E_{\hat{U}_{12S}}[{\tilde{g}_{12S}}]
		+
		(1 - \alpha^*)\sqrt{M}E_{\hat{U'}_{12S}}[{\tilde{g}_{12S}}]
		-
		(\alpha^*\sqrt{M}E_{U_{12S}}[{\tilde{g}_{12S}}]
		\\&\quad+
		(1 - \alpha^*)\sqrt{M}E_{U'_{12S}}[{\tilde{g}_{12S}}]
		)
		\\ &=
		\alpha^*\sqrt{M/n} \sqrt{n}(
		E_{\hat U_{12S}}[{\tilde{g}_{12S}}]
		-
		E_{U_{12S}}[{\tilde{g}_{12S}}])
		+
		(1 - \alpha^*)\sqrt{M/n'} \sqrt{n'}(
		E_{\hat U'_{12S}}[{\tilde{g}_{12S}}]
		-
		E_{U'_{12S}}[{\tilde{g}_{12S}}])
		\\ &\overset{d}{\rightarrow}
		\mathcal{N}(0, \nu {\alpha^*}^2 V_{U_{12S}}[\tilde{g}_{12S}] + \nu' (1 - {\alpha^*})^2 V_{U'_{12S}}[\tilde{g}_{12S}]). 
	\end{align*}
	
	The remaining part is $d_0$. We have
	\begin{align*}
		d_0&=\frac{\partial m}{\partial \alpha}(\alpha^*)=E_{U_{12}}[\tilde{g}_{12S}]-E_{U'_{12S}}[\tilde{g}_{12S}]+E_{F^{\alpha^*}_{12S}}[\frac{\partial}{\partial\alpha}(g_1-\mu_1^\alpha)(g_2-\mu_2^\alpha)|_{\alpha=\alpha^*}],
	\end{align*}
	where we interchange differentiation and integration, assuming $\frac{\partial}{\partial \alpha}\mu^\alpha_1$ and $\frac{\partial}{\partial \alpha}\mu^\alpha_2$ are bounded.
	
	Let us evaluate the derivative term inside the expectation:
	\begin{align*}
		\frac{\partial}{\partial\alpha}(g_1-\mu_1^\alpha)(g_2-\mu_2^\alpha)|_{\alpha=\alpha^*}=&-g_1 \frac{\partial}{\partial\alpha}\mu^{\alpha^*}_2-g_2 \frac{\partial}{\partial\alpha}\mu^{\alpha^*}_1+\mu^{\alpha^*}_1\frac{\partial}{\partial\alpha}\mu^{\alpha^*}_2+\mu^{\alpha^*}_2\frac{\partial}{\partial\alpha}\mu^{\alpha^*}_1\\
		=&(\mu^{\alpha^*}_1-g_1) \frac{\partial}{\partial\alpha}\mu^{\alpha^*}_2+(\mu^{\alpha^*}_2-g_2)\frac{\partial}{\partial\alpha}\mu^{\alpha^*}_1
	\end{align*}
	and then $E_{F^{\alpha^*}_{12S}}[\frac{\partial}{\partial\alpha}(g_1-\mu_1^\alpha)(g_2-\mu_2^\alpha)|_{\alpha=\alpha^*}]=0$. 	
	
	Therefore, the expression for $d_0$ is simplified to:
		$$d_0 = E_{U_{12S}}[\tilde{g}_{12S}]-E_{U'_{12S}}[\tilde{g}_{12S}]=
		 \begin{cases}
			-(\theta - \theta')E_{F^{\bar\alpha^*}_{12S}}[\tilde{g}_{12S}] & \text{if $\alpha^*=\alpha_+$,} \\
			(\theta - \theta')E_{F^{\bar\alpha^*}_{12S}}[\tilde{g}_{12S}]       & \text{if $\alpha^*=\alpha_-$.}
		\end{cases}$$

\end{proof}

\section{WEAKLY-SUPERVISED KERNEL CI AND MCI TEST WITH TRUE MIXTURE PROPORTIONS}\label{appsec:test_with_prop}
In the proofs of this section, we denote $F_\tau^{\alpha^*}$ as $F_\tau$ for simplicity.

\subsection{Proofs for the CI test}\label{appsubsec:proof_ci_test}
\begin{proof}[Proof of Theorem \ref{thm:tci_asyd}]
		
	We define the centralized kernel $\tilde k_{12}$ associated with the feature map $\tilde{\varphi}_{12}(x):= (\varphi_1(x_1) - E_{F_1}[\varphi_1(x_1)]) \otimes (\varphi_2(x_2) - E_{F_2}[\varphi_2(x_2)])$. Then,
	\begin{align*}
		\tilde k_{12}(x,x')&=\left(k_1(x_1,x'_1)-E_{z_1\sim F_1}k_1(x_1,z_1)-E_{z_1\sim F_1}k_1(x'_1,z_1)+E_{z_1,z'_1\sim F_1}k_1(z_1,z'_1)\right)\\
		&\left(k_2(x_2,x'_2)-E_{z_2\sim F_2}k_2(x_2,z_2)-E_{z_2\sim F_2}k_2(x'_2,z_2)+E_{z_2,z'_2\sim  F_2}k_2(z_2,z'_2)\right). 
	\end{align*}
	Since $k_1$ and $k_2$ are positive-definite kernels,  by Mercer's theorem \citep{schol01kernel}, they can be expanded as  $k_1(x_1,x'_1)=\Sigma_{r=1}^\infty\lambda_{1,r}\phi_{1,r}(x_1)\phi_{1,r}(x'_1)$ and $k_2(x_2,x'_2)=\Sigma_{r=1}^\infty\lambda_{2,r}\phi_{2,r}(x_2)\phi_{2,r}(x'_2)$ where $\lambda_{1, r},\lambda_{2, r}$ and $\phi_{1,r}, \phi_{2,r}$ are eigenvalues and eigenfunctions. Since these expansions are absolutely convergent, applying Fubini-Tonelli theorem, we can write $\tilde k_{12}(x,x')$ as follows:
	\begin{align*}
		\tilde k_{12}(x,x')=&\left(k_1(x_1,x'_1)-E_{z_1\sim F_1}k_1(x_1,z_1)-E_{z_1\sim F_1}k_1(x'_1,z_1)+E_{z_1,z'_1\sim F_1}k_1(z_1,z'_1)\right)\\
		&\left(k_2(x_2,x'_2)-E_{z_2\sim F_2}k_2(x_2,z_2)-E_{z_2\sim F_2}k_2(x'_2,z_2)+E_{z_2,z'_2\sim  F_2}k_2(z_2,z'_2)\right)\\		
		=&\left(\sum_{r=1}^\infty\lambda_{1,r}\left(\phi_{1,r}(x_1)\phi_{1,r}(x'_1)-\phi_{1,r}(x_1)E_{F_1}\phi_{1,r}(z_1)-\phi_{1,r}(x'_1)E_{F_1}\phi_{1,r}(z_1)+E^2_{F_1}\phi_{1,r}(z_1)\right)\right)\\
		&\left(\sum_{r=1}^\infty\lambda_{2,r}\left(\phi_{2,r}(x_2)\phi_{2,r}(x'_2)-\phi_{2,r}(x_2)E_{F_2}\phi_{2,r}(z_2)-\phi_{2,r}(x'_2)E_{F_2}\phi_{2,r}(z_2)+E^2_{F_2}\phi_{2,r}(z_2)\right)\right)\\
		=&\left(\sum_{r=1}^\infty\lambda_{1,r}\left(\phi_{1,r}(x_1)-E_{F_1}\phi_{1,r}(z_1)\right)\left(\phi_{1,r}(x'_1)-E_{F_1}\phi_{1,r}(z_1)\right)\right)\\
		&\left(\sum_{r=1}^\infty\lambda_{2,r}\left(\phi_{2,r}(x_2)-E_{F_2}\phi_{2,r}(z_2)\right)\left(\phi_{2,r}(x'_2)-E_{F_2}\phi_{2,r}(z_2)\right)\right)\\
		=&\left(\sum_{r=1}^\infty\lambda_{1,r}\tilde\phi_{1,r}(x_1)\tilde\phi_{1,r}(x'_1)\right)\left(\sum_{r=1}^\infty\lambda_{2,r}\tilde\phi_{2,r}(x_2)\tilde\phi_{2,r}(x'_2)\right)\\
		=& \sum_{i,j=1}^{\infty} \lambda_{1,i}\lambda_{2,j}  \tilde\phi_{1,i}(x_1) \tilde\phi_{1,i}(x'_1)\tilde\phi_{2,j}(x_2) \tilde\phi_{2,j}(x'_2), 
	\end{align*}
	where we define $\tilde\phi_{1,r}(x_1)=\phi_{1,r}(x_1)-E_{F_1}\phi_{1,r}(z_1)$ and $\tilde\phi_{2,r}(x'_2)=\phi_{2,r}(x'_2)-E_{F_2}\phi_{2,r}(z_2)$.
	
	Then the test statistic $T_{CI}$ is written as follows with $\tilde\phi_{1,r}$ and $\tilde\phi_{2,r}$:
	\begin{align*}
		T_{CI}&=\left\|E_{\hat F_{12}}\left[\varphi_1\otimes\varphi_2\right]-E_{\hat F_{1}\hat F_{2}}\left[\varphi_1\otimes\varphi_2\right]\right\|^2_{\mathcal{H}}\\
		&=\left\|E_{\hat F_{12}}\left[(\varphi_1-E_{F_1}\varphi_1)\otimes(\varphi_2-E_{F_2}\varphi_2)\right]-E_{\hat F_{2}\hat F_{2}}\left[(\varphi_1-E_{F_1}\varphi_1)\otimes(\varphi_2-E_{F_2}\varphi_2)\right]\right\|^2_{\mathcal{H}}\\
		&=E_{\hat{F}_{12},\hat{F}_{12}} \tilde k_{12}(x,x')
		-2E_{\hat{F}_{12},\hat{F}_1\hat F_2}\tilde k_{12}(x,x')+E_{\hat{F}_{1}\hat{F}_{2},\hat{F}_{1}\hat{F}_{2}} \tilde k_{12}(x,x')\\
		&=\sum_{i,j=1}^{\infty}\lambda_{1,i}\lambda_{2,j} (E_{\hat F_{12}, \hat F_{12}}\left[\tilde\phi_{1,i}(x_1) \tilde\phi_{1,i}(x'_1)\tilde\phi_{2,j}(x_2) \tilde\phi_{2,j}(x'_2)\right]-2E_{\hat{F}_{12},\hat{F}_1\hat F_2}\left[\cdots\right]\\
		&\quad+E_{\hat{F}_1\hat F_2,\hat{F}_1\hat F_2}\left[\cdots\right])\\			
		&=\sum_{i,j=1}^{\infty}\lambda_{1,i}\lambda_{2,j}\left(E_{\hat F_{12}}\left[\tilde\phi_{1,i}(x_1) \tilde\phi_{2,j}(x_2)\right]-E_{\hat F_1\hat F_2}\left[\tilde\phi_{1,i}(x_1) \tilde\phi_{2,j}(x_2)\right]\right)^2. 
	\end{align*}
	
	Now we consider the asymptotic distribution of $MT_{CI}$ under $H_0$.  $MT_{CI}$ can be written as
	\begin{align*}
		MT_{CI}&=\sum_{i,j=1}^{\infty}\lambda_{1,i}\lambda_{2,j}\left(\sqrt{M}\alpha^*E_{\hat U_{12}}\left[\tilde\phi_{1,i}(x_1) \tilde\phi_{2,j}(x_2)\right]+\sqrt{M}(1-\alpha^*)E_{\hat U'_{12}}\left[\cdots\right]\right.\\
		&-\left(\sqrt{M}\alpha^*E_{\hat U_{1}}\left[\tilde\phi_{1,i}(x_1)\right] +\sqrt{M}(1-\alpha^*)E_{\hat U'_{1}}\left[\cdots\right]\right) \\
		&\quad\left.\left(\alpha^*E_{\hat U_{2}}\left[\tilde\phi_{2,j}(x_2)\right]+(1-\alpha^*)E_{\hat U'_{2}}\left[\cdots\right]\right)\right)^2. 
	\end{align*}
	
	We denote $T^L_{CI}$ as the partial sum of $T_{CI}$ up to $L$-th eigenvalues and then
	\begin{align*}
		MT^L_{CI}&=\sum_{i,j=1}^{L} \lambda_{1,i}\lambda_{2,j}\left(
		\underset{(a)}{\underline{\sqrt\frac{M}{n}\sqrt{n}\alpha^*E_{\hat U_{12}}\left[\tilde\phi_{1,i}(x_1) \tilde\phi_{2,j}(x_2)\right]+\sqrt\frac{M}{n'}\sqrt{n'}(1-\alpha^*)E_{\hat U'_{12}}\left[\tilde\phi_{1,i}(x_1) \tilde\phi_{2,j}(x_2)\right]}}\right.\\
		&-\left(\underset{(b)}{\underline{\sqrt\frac{M}{n}\sqrt{n}\alpha^*E_{\hat U_{1}}\left[\tilde\phi_{1,i}(x_1)\right] +\sqrt\frac{M}{n'}\sqrt{n'}(1-\alpha^*)E_{\hat U'_{1}}\left[\tilde\phi_{1,i}(x_1)\right]}}\right)\\
		&\left.\left(\underset{(c)}{\underline{\alpha^*E_{\hat U_{2}}\left[\tilde\phi_{2,j}(x_2)\right]+(1-\alpha^*)E_{\hat U'_{2}}\left[\tilde\phi_{2,j}(x_2)\right]}}\right)\right)^2. 
	\end{align*}
	Under $H_0$, $(a)$ and $(b)$ are written as:
	\begin{align*}
		(a)&=\sqrt\frac{M}{n}\sqrt{n}\alpha^*\left(E_{\hat U_{12}}\left[\tilde\phi_{1,i}(x_1) \tilde\phi_{2,j}(x_2)\right]-E_{U_{12}}[\tilde\phi_{1,i}(x_1) \tilde\phi_{2,j}(x_2)]\right)\\
		&\quad+\sqrt\frac{M}{n'}\sqrt{n'}(1-\alpha^*)\left(E_{\hat U'_{12}}[\tilde\phi_{1,i}(x_1) \tilde\phi_{2,j}(x_2)]-E_{U'_{12}}[\tilde\phi_{1,i}(x_1) \tilde\phi_{2,j}(x_2)]\right),\\
		(b)&=\sqrt\frac{M}{n}\sqrt{n}\alpha^*\left(E_{\hat U_{1}}\left[\tilde\phi_{1,i}(x_1)\right]-E_{U_{1}}[\tilde\phi_{1,i}(x_1)]\right)+\sqrt\frac{M}{n'}\sqrt{n'}(1-\alpha^*)\left(E_{\hat U'_{1}}[\tilde\phi_{1,i}(x_1)]-E_{U'_{1}}[\tilde\phi_{1,i}(x_1)]\right). 
	\end{align*}
	
	As $M\rightarrow\infty$, $(a)$ and $(b)$ converge to normal distributions from the CLT, while $(c)\xrightarrow{p}0$ from the LLN. Combining these results, it follows that
	$$MT_{CI}^L\xrightarrow{d}\sum_{i,j=1}^{L} \lambda_{1,i}\lambda_{2,j}\xi^2_{i,j},$$
	where $(\xi_{1,1}\dots\xi_{L,L})$ follows a multivariate normal distribution with a mean $\mathbf{0}$ and following covariances: $\forall i,j,i',j'\in[L]$,
	\begin{align*}
		Cov[\xi_{i,j},\xi_{i',j'}]&=\nu{\alpha^*}^2Cov_{U_{12}}[\tilde\phi_{1,i}(x_1)\tilde\phi_{2,j}(x_2),\tilde\phi_{1,i'}(x_1)\tilde\phi_{2,j'}(x_2)]\\
		&+\nu'(1-\alpha^*)^2Cov_{U'_{12}}[\tilde\phi_{1,i}(x'_1)\tilde\phi_{2,j}(x'_2),\tilde\phi_{1,i'}(x'_1)\tilde\phi_{2,j'}(x'_2)]. 
	\end{align*}
	

	We now derived the asymptotic distribution of $MT^L_{CI}$. To derive the asymptotic distribution of $MT_{CI}$, we follow a similar procedure to the section 5.5.2 of \citet{serfling81app} . Then, we can show that the expectation of difference between $MT_{CI}$ and $MT^L_{CI}$ vanishes:
	\begin{align*}
		E\left[M|T_{CI}-T^L_{CI}|\right]&=E\left[M\sum^\infty_{i,j>L}\lambda_{1,i}\lambda_{2,j}\left(E_{\hat F_{12}}\left[\tilde\phi_{1,i}(x_1) \tilde\phi_{2,j}(x_2)\right]-E_{\hat F_1\hat F_2}\left[\tilde\phi_{1,i}(x_1) \tilde\phi_{2,j}(x_2)\right]\right)^2\right]\rightarrow0,\\
	\end{align*}
	as $L\rightarrow\infty$. 
	
	Next, we denote the limiting variable of $MT^L_{CI}$ as $W_{\lim}^L:=\sum_{i,j=1}^{L} \lambda_{1,i}\lambda_{2,j}\xi^2_{i,j},$
	and define $W_{\lim}:=\lim_{L\rightarrow\infty}W_{\lim}^L$. Then, we can show 
	\begin{align*}
		E\left[|W_{\lim}^L-W_{\lim}|\right]=E\left[\sum_{i,j>L}^{\infty}\lambda_{1,i}\lambda_{2,j}\xi^2_{i,j}\right]=\sum_{i,j>L}^{\infty}\lambda_{1,i}\lambda_{2,j}E\left[\xi^2_{i,j}\right]\rightarrow 0
	\end{align*}
	as $L\rightarrow\infty$.
	These results allow us to prove the pointwise convergence of the characteristic functions. Specifically, for any $t$ and $\epsilon>0$, and for all sufficiently large $M$ and $L$, 
	\begin{align*}
		\bigl| E\left[e^{itMT_{CI}}\right]&-E\left[e^{itW_{\lim}}\right]\bigr|\\
		=&\left|\left(E\left[e^{itMT_{CI}}\right]-E\left[e^{itMT^L_{CI}}\right]\right)+\left(E\left[e^{itMT^L_{CI}}\right]-E\left[e^{itW^L_{\lim}}\right]\right)+\left(E\left[e^{itW^L_{\lim}}\right]-E\left[e^{itW_{\lim}}\right]\right)\right|\\
		\leq&\left|E\left[e^{itMT_{CI}}\right]-E\left[e^{itMT^L_{CI}}\right]\right|+\left|E\left[e^{itMT^L_{CI}}\right]-E\left[e^{itW^L_{\lim}}\right]\right|+\left|E\left[e^{itW^L_{\lim}}\right]-E\left[e^{itW_{\lim}}\right]\right]\\
		 \leq& |t| E\left[M\left|T_{C I} - T_{C I}^L\right|\right] + \left|E\left[e^{i t M T_{C I}^L}\right]-E\left[e^{i t W_{\mathrm{lim}}^L}\right]\right| + |t| E\left[\left|W_{\mathrm{lim}}^L - W_{\mathrm{lim}}\right|\right] \leq \epsilon,
	\end{align*}
	where we used the inequality $\left|e^{i z}-1\right| \leq|z|$. Thus,  the asymptotic distribution of $MT_{CI}$ is:
	$$MT_{CI}\xrightarrow{d}W_{\lim}=\sum_{i,j=1}^{\infty}\lambda_{1,i}\lambda_{2,j}\xi^2_{i,j},$$	
	where  $\xi_{i,j}$'s follow the multivariate normal distribution defined above.
	
	We next consider the asymptotic behavior of $MT_{CI}$ under $H_1$. In this case, from Theorem 2 in \citet{gretton15simpler}, the population version of $T_{CI}$ equals some positive value $c$. Since $T_{CI}$ is a two-sample V-statistic and a consistent estimator \citep{huang23weighted}, $T_{CI}\xrightarrow{p}c$, as $M\rightarrow\infty$. 
	Therefore, $MT_{CI}\xrightarrow{p}\infty$, as $M\rightarrow\infty$.
\end{proof}

\begin{proof}[Proof of Theorem \ref{thm:tci_mv_w/alpha}]
	In this proof, we use the property of V-statistics. $T_{CI}$ is a two-sample V-statistic \citep{Vaart98asymp} since it can be written as follows.
	\begin{align*}
		T_{CI}=&\frac{1}{n^6n'^6}\sum^{n}_{i_1,...,i_6=1}\sum^{n'}_{q_1,...,q_6=1}h_{i_1,...,i_6,q_1,...,q_6}
	\end{align*}
	where $h_{i_1,...,i_6,q_1,...,q_6}$ is a symmetric function such that
	\begin{align*}
		h_{i_1,...,i_6,q_1,...,q_6}:=&\frac{1}{6!6!}\sum^{(i_1,..,i_6)}_{(j_1,...,j_6)}\sum^{(q_1,..,q_6)}_{(r_1,...,r_6)}\left<\varphi_{j_1,...,j_3,r_1,...,r_3}, \varphi_{j_4,...,j_6,r_4,...,r_6}\right>		
	\end{align*}
	
	and 
	\begin{align*}
		\varphi&_{j_1,...,j_3,r_1,...,r_3}:=\\
		&\alpha^*\tilde\varphi_{12}(x^{(j_1)})+(1-\alpha^*)\tilde\varphi_{12}(x'^{(r_1)})-(\alpha^*\tilde\varphi_{1}(x^{(j_2)}_1)+(1-\alpha^*)\tilde\varphi_{1}(x'^{(r_2)}_1))\otimes(\alpha^*\tilde\varphi_{2}(x^{(j_3)}_2)+(1-\alpha^*)\tilde\varphi_{2}(x'^{(r_3)}_2)).
	\end{align*}
	
	Here, the summation $\sum_{\left(j_1,... ,j_6\right)}^{\left(i_1,... ,i_6\right)}$ is taken over all ordered $(j_1,...,j_6)$ drawn without replacement from $(i_1,...,i_6)$.
	
	This is a degenerate V-statistic, which means 
	$$E_{i_2,...,i_6,q_1,...,q_6}[h_{i_1,...,i_6,q_1,...,q_6}]=E_{i_1,...,i_6,q_2,...,q_6}[h_{i_1,...,i_6,q_1,...,q_6}]=0$$
	where we take the expectations by samples of each index in the subscripts. 
	
	Next, we define a related statistic, $\check T_{CI}:=E_{x,x'\sim\hat F_{12}}[\tilde k_{12}(x,x')]$. We prove the limit mean and variance of $MT_{CI}$ and $M\check T_{CI}$ are the same, and derive the mean and variance of $M\check T_{CI}$.
	 
	$\check T_{CI}$ is a two-sample V-statistics and written as 
	\begin{align*}
		\check T_{CI}=&\frac{1}{n^2n'^2}\sum^{n}_{i_1,i_2=1}\sum^{n'}_{q_1,q_2=1}\check h_{i_1,i_2,q_1,q_2}
	\end{align*}
	where $\check h_{i_1,i_2,q_1,q_2}$ is a symmetric function such that
	\begin{align*}
		\check h_{i_1,i_2,q_1,q_2}:=&\frac{1}{2!2!}\sum^{(i_1,i_2)}_{(j_1,j_2)}\sum^{(q_1,q_2)}_{(r_1,r_2)}\langle\check\varphi_{j_1,r_1}, \check\varphi_{j_2,r_2}\rangle
	\end{align*}
	and $\check\varphi_{j_1,r_1}:=\alpha^*\tilde\varphi_{12}(x^{(j_1)})+(1-\alpha^*)\tilde\varphi_{12}(x'^{(r_1)})$. $\check T_{CI}$ is also degenerate since $E_{i_2,q_1,q_2}[\check h_{i_1,i_2,q_1,q_2}]=E_{i_1,i_2,q_2}[\check h_{i_1,i_2,q_1,q_2}]=0$.
	
	Furthermore, we consider the difference $T_{CI}-\check T_{CI}$, which itself is a V-statistic:
	\begin{align*}
		T_{CI}-\check T_{CI}=&\left<2E_{\hat F_{12}}[\tilde\varphi_{12}(x)]-E_{\hat F_{1}}[\tilde\varphi_{1}(x_1)]\otimes E_{\hat F_{2}}[\tilde\varphi_{2}(x_2)],-E_{\hat F_{1}}[\tilde\varphi_{1}(x_1)]\otimes E_{\hat F_{2}}[\tilde\varphi_{2}(x_2)]\right>\\
		=&\frac{1}{n^5n'^5}\sum^{n}_{i_1,...,i_5=1}\sum^{n'}_{q_1,...,q_5=1}\bar h_{i_1,...,i_5,q_1,...,q_5}
	\end{align*}
	
	where $\bar h_{i_1,...,i_5,q_1,...,q_5}$ is a symmetric function such that
	
	$$\bar h_{i_1,...,i_5,q_1,...,q_5}:=\frac{1}{5!5!}\sum^{(i_1,..,i_5)}_{(j_1,...,j_5)}\sum^{(q_1,..,q_5)}_{(r_1,...,r_5)}\left<\bar\varphi_{j_1,...,j_3,r_1,...,r_3}, \bar\varphi_{j_4,j_5,r_4,r_5}\right>$$
	
	and we define
	\begin{align*}
		\bar\varphi&_{j_1,...,j_3,r_1,...,r_3}:=\\
		&2\alpha^*\tilde\varphi_{12}(x^{(j_1)})+2(1-\alpha^*)\tilde\varphi_{12}(x'^{(r_1)})-(\alpha^*\tilde\varphi_{1}(x_1^{(j_2)})+(1-\alpha^*)\tilde\varphi_{1}(x'^{(r_2)}_1))\otimes(\alpha^*\tilde\varphi_{2}(x^{(j_3)}_2)+(1-\alpha^*)\tilde\varphi_{2}(x'^{(r_3)}_2))
	\end{align*}
	and 	
	$$\bar\varphi_{j_4,j_5,r_4,r_5}:=-(\alpha^*\tilde\varphi_{1}(x_1^{(j_4)})+(1-\alpha^*)\tilde\varphi_{1}(x'^{(r_4)}_1))\otimes(\alpha^*\tilde\varphi_{2}(x^{(j_5)}_2)+(1-\alpha^*)\tilde\varphi_{2}(x'^{(r_5)}_2)).$$

	We next analyze the expectation $E[T_{CI}-\check T_{CI}]$. Since $E[\left<\bar\varphi_{j_1,...,j_3,r_1,...,r_3}, \bar\varphi_{j_4,j_5,r_4,r_5}\right>]=0$ unless at least two pairs of indices in $j_1,...,j_5,r_1,...,r_5$ are equivalent, 
	
	$$E[T_{CI}-\check T_{CI}]=\frac{1}{n^5n'^5}O(M^8).$$
	
	Thus,  $ME[T_{CI}-\check T_{CI}]\rightarrow0$ as $M\rightarrow\infty$. Since the expectations of $MT_{CI}$ and $M\check T_{CI}$ are asymptotically equivalent, we now focus on analyzing $E[\check{T}_{CI}]$, which is easier to obtain. 
	
	$E[\check T_{CI}]$ is derived by a similar approach to the estimation of $E[HSIC_b]$ in \citet{gretton07kernal}. First, the V-statistic $\check T_{CI}$ is expanded as 
	$$\check T_{CI}=\frac{{\alpha^*}^2}{n^2}\sum^n_{i_1,i_2=1}\tilde k_{12}(x^{(i_1)},x^{(i_2)})+\frac{(1-\alpha^*)^2}{n'^2}\sum^{n'}_{q_1,q_2=1}\tilde k_{12}(x'^{(q_1)},x'^{(q_2)})+2\frac{\alpha^*(1-\alpha^*)}{nn'}\sum^n_{i_1=1}\sum^{n'}_{q_1=1}\tilde k_{12}(x^{(i_1)},x'^{(q_1)}).$$
	
	Next, we define the corresponding U-statistic for $\check T_{CI}$ as $\check T_{CI,U}$:
	$$\check T_{CI,U}=\frac{{\alpha^*}^2}{(n)_2}\sum_{i_1\neq i_2}\tilde k_{12}(x^{(i_1)},x^{(i_2)})+\frac{(1-\alpha^*)^2}{(n')_2}\sum_{q_1\neq q_2}\tilde k_{12}(x'^{(q_1)},x'^{(q_2)})+2\frac{\alpha^*(1-\alpha^*)}{nn'}\sum^n_{i_1=1}\sum^{n'}_{q_1=1}\tilde k_{12}(x^{(i_1)},x'^{(q_1)}), $$
	where $(n)_m:=\frac{n!}{(n-m)!}$. Note that the U-statistic, $\check T_{CI,U}$ is an unbiased estimator of its population mean, and $E[\check T_{CI,U}]=0$.
	
	The difference between the V-statistic and the U-statistic is given by:
	\begin{align*}
		\check T_{CI}-\check T_{CI,U}&=\frac{{\alpha^*}^2}{n^2}\sum^n_{i_1=1}\tilde k_{12}(x^{(i_1)},x^{(i_1)})-\frac{{\alpha^*}^2}{n(n)_2}\sum_{i_1\neq i_2}\tilde k_{12}(x^{(i_1)},x^{(i_2)})\\
		&+\frac{(1-\alpha^*)^2}{{n'}^2}\sum^{n'}_{q_1=1}\tilde k_{12}(x'^{(q_1)},x'^{(q_1)})-\frac{(1-\alpha^*)^2}{n'(n')_2}\sum_{q_1\neq q_2}\tilde k_{12}(x'^{(q_1)},x'^{(q_2)}). 
	\end{align*}
	
	Since $E[\check T_{CI,U}]=0$, taking the expectation of the equation above yields the desired result:
	\begin{align*}
		E[\check{T}_{CI}]&=E[\check{T}_{CI}-\check{T}_{CI,U}]\\
		&=\frac{{\alpha^*}^2}{n}E_{x^{(i_1)},x^{(i_2)}\sim U_{12}}[\tilde k_{12}(x^{(i_1)},x^{(i_1)})-\tilde 	k_{12}(x^{(i_1)},x^{(i_2)})]\\&+\frac{(1-\alpha^*)^2}{n'}E_{x'^{(q_1)},x'^{(q_2)}\sim U'_{12}}[\tilde k_{12}(x'^{(q_1)},x'^{(q_1)})-\tilde k_{12}(x'^{(q_1)},x'^{(q_2)})],
	\end{align*}
	from which the limit of $E[MT_{CI}]$ is obtained.
	
	
	Next, we derive $V[T_{CI}]$. Using the expression of the V-statistic, we have:
	$$V[T_{CI}]=\frac{1}{n^{12}n'^{12}}\sum^n_{i_1,...,i_6,i'_1,...,i'_6=1}\sum^{n'}_{q_1,...,q_6,q'_1,...,q'_6=1}Cov[h_{i_1,...,i_6,q_1,...,q_6},h_{i'_1,...,i'_6,q'_1,...,q'_6}].$$
	
	Recall that $T_{CI}$ is a degenerate V-statistic and $E_{i_2,...,i_6,q_1,...,q_6}[h_{i_1,...,i_6,q_1,...,q_6}]=E_{i_1,...,i_6,q_2,...,q_6}[h_{i_1,...,i_6,q_1,...,q_6}]=0$. Thus, in order for $Cov[h_{i_1,...,i_6,q_1,...,q_6},h_{i'_1,...,i'_6,q'_1,...,q'_6}]$ to be nonzero, at least two pairs of indices must be identical between the sets $\{i_1,...,i_6,q_1,...,q_6\}$ and $\{i'_1,...,i'_6,q'_1,...,q'_6\}$. Using this combinatorial constraint, we can identify the leading order terms as $M \rightarrow \infty$.

	We restrict our focus to combinations where exactly two indices overlap, as sharing more variables reduces the free choices from the sample, making those terms asymptotically negligible. Specifically, there are three cases for sharing exactly two variables: (1) two shared $i$ indices and zero shared $q$ indices, (2) zero shared $i$ indices and two shared $q$ indices, and (3) one shared $i$ index and one shared $q$ index.

	In case (1), for the $i$ indices, we choose 6 distinct indices for the first $h$ function ($(n)_{6}$ ways), select 2 of these to share ($\mathrm{C}^{6}_{2}$ ways), arrange them in the second $h$ function ($6\times5$ ways), and fill its remaining 4 slots with unselected indices ($(n-6)_{4}$ ways). This yields a total of $\mathrm{C}^{6}_{2}\cdot6\cdot5(n)_{10}(n')_{12}$ combinations. By symmetry, case (2) yields $\mathrm{C}^{6}_{2}\cdot6\cdot5(n)_{12}(n')_{10}$ combinations. For case (3), sharing one $i$ index and one $q$ index yields $(n)_{6}\times6\times6\times(n-6)_{5}\times(n')_{6}\times6\times6\times(n'-6)_{5}=6^{2}\cdot6^{2}(n)_{11}(n')_{11}$ combinations. By considering only these leading order terms, we obtain:
	\begin{align*}
		V[ T_{CI}]&=\frac{1}{n^{12}n'^{12}}(\mathrm{C}^6_2\cdot6\cdot5(n)_{10}(n')_{12}Cov[h_{i_1,...,i_6,q_1,...,q_6},h_{i_1,i_2,i'_3,...,i'_6,q'_1,...,q'_6}]\\
		&+\mathrm{C}^6_2\cdot6\cdot5(n)_{12}(n')_{10}Cov[h_{i_1,...,i_6,q_1,...,q_6},h_{i'_1,...,i'_6,q_1,q_2,q'_3,...,q'_6}]\\
		&+6^2\cdot6^2(n)_{11}(n')_{11}Cov[h_{i_1,...,i_6,q_1,...,q_6},h_{i_1,i'_2,...,i'_6,q_1,q'_2,...,q'_6}]+O(M^{21}))\\
		&=\frac{1}{n^{12}n'^{12}}(\mathrm{C}^6_2\cdot6\cdot5(n)_{10}(n')_{12}E_{i_1,i_2}[(E_{i_3,...,i_6,q_1,...,q_6}[h_{i_1,...,i_6,q_1,...,q_6}])^2]\\
		&+\mathrm{C}^6_2\cdot6\cdot5(n)_{12}(n')_{10}E_{q_1,q_2}[(E_{i_1,...,i_6,q_3,...,q_6}[h_{i_1,...,i_6,q_1,...,q_6}])^2]\\
		&+6^2\cdot6^2(n)_{11}(n')_{11}E_{i_1,q_1}[(E_{i_2,...,i_6,q_2,...,q_6}[h_{i_1,...,i_6,q_1,...,q_6}])^2]+O(M^{21})).
	\end{align*}
	
	We simplify the each term in the above equation. Recall that $h_{i_1,...,i_6,q_1,...,q_6}$ is the sum of inner products $\left<\varphi_{j_1,...,j_3,r_1,...,r_3}, \varphi_{j_4,...,j_6,r_4,...,r_6}\right>$. Since the feature map $\varphi_{j_1,...,j_3,r_1,...,r_3}$ are centered, the expectation vanishes if all sample indices are distinct. Taking this into account, we obtain
	\begin{align*}
		E_{i_1,i_2}[(E_{i_3,...,i_6,q_1,...,q_6}[h_{i_1,...,i_6,q_1,...,q_6}])^2]&=(\frac{1}{6!6!}\cdot2\cdot4!6!)^2E_{i_1,i_2}[(E_{q_1,q_2}[\langle\check\varphi_{i_1,q_1},\check\varphi_{i_2,q_2}\rangle])^2]\\
		E_{q_1,q_2}[(E_{i_1,...,i_6,q_3,...,q_6}[h_{i_1,...,i_6,q_1,...,q_6}])^2]&=(\frac{1}{6!6!}\cdot2\cdot4!6!)^2E_{q_1,q_2}[(E_{i_1,i_2}[\langle\check\varphi_{i_1,q_1},\check\varphi_{i_2,q_2}\rangle])^2]\\
		E_{i_1,q_1}[(E_{i_2,...,i_6,q_2,...,q_6}[h_{i_1,...,i_6,q_1,...,q_6}])^2]&=(\frac{1}{6!6!}\cdot2\cdot5!5!)^2E_{i_1,q_2}[(E_{i_2,q_1}[\langle\check\varphi_{i_1,q_1},\check\varphi_{i_2,q_2}\rangle])^2].
	\end{align*}

	Substituting these expressions into the formula for $V[T_{CI}]$ and taking the limit as $M\rightarrow\infty$, we derive the desired asymptotic variance:
	\begin{align*}
		V[MT_{CI}]=M^2 V[T_{CI}]\rightarrow& 2\nu^2E_{i_1,i_2}[(E_{q_1,q_2}[\langle\check\varphi_{i_1,q_1},\check\varphi_{i_2,q_2}\rangle])^2]+2\nu'^2E_{q_1,q_2}[(E_{i_1,i_2}[\langle\check\varphi_{i_1,q_1},\check\varphi_{i_2,q_2}\rangle])^2]\\
		&+4\nu\nu'E_{i_1,q_2}[(E_{i_2,q_1}[\langle\check\varphi_{i_1,q_1},\check\varphi_{i_2,q_2}\rangle])^2].
	\end{align*}
\end{proof}

\subsection{Proofs for the MCI test}\label{appsubsec:proof_mci_test}
\begin{proof}[Proof of Theorem \ref{thm:tmci_asyd}]  
	We begin by defining the centered kernels $\tilde k_{1S}(x_{1S},x'_{1S})=\langle 
	\varphi_1(x_1)-\mu_{X_1 \mid X_S}(x_S)
	,\varphi_1(x'_1)-\mu_{X_1 \mid X_S}(x'_S)
	\rangle$ and $\tilde k_{2S}(x_{2S},x'_{2S})=\langle 
	\varphi_2(x_2)-\mu_{X_2 \mid X_S}(x_S)
	,\varphi_2(x'_2)-\mu_{X_2 \mid X_S}(x'_S)
	\rangle$.
	By Mercer's theorem, these kernels can be expanded 
	\begin{align*}
		\tilde k_{1S}(x_{1S},x'_{1S})&=\sum^\infty_{r=1}\lambda_{1,r}(\phi_{1,r}(x_1)-E_F[\phi_{1,r}(x_1)|x_S])(\phi_{1,r}(x'_1)-E_F[\phi_{1,r}(x_1)|x'_S]) \\
		&=\sum^\infty_{r=1}\lambda_{1,r}\tilde\phi_{1,r}(x_{1S})\tilde\phi_{1,r}(x'_{1S}),\\
		\tilde k_{2S}(x_{2S},x'_{2S})&=\sum^\infty_{r=1}\lambda_{2,r}(\phi_{2,r}(x_2)-E_F[\phi_{2,r}(x_2)|x_S])(\phi_{2,r}(x'_2)-E_F[\phi_{2,r}(x_2)|x'_S]) \\
		&=\sum^\infty_{r=1}\lambda_{2,r}\tilde\phi_{2,r}(x_{2S})\tilde\phi_{2,r}(x'_{2S}), \\
		k_S(x_{S},x'_{S})&=\sum^\infty_{r=1}\lambda_{S,r}\phi_{S,r}(x_S)\phi_{S,r}(x'_S),
	\end{align*}
	where $\lambda_{S,r}$ and $\phi_{S,r}$ are eigenvalues and eigenfunctions of the operator associated to $k_S$. We also define centered eigenfunctions $\tilde\phi_{1,r}(x_{1S}):=\phi_{1,r}(x_1)-E_F[\phi_{1,r}(x_1)|x_S]$ and $\tilde\phi_{2,r}(x_{2S}):=\phi_{2,r}(x_2)-E_F[\phi_{2,r}(x_2)|x_S]$ in this proof.
	
	Then $MT_{MCI}$ can be expressed as 
	\begin{align*}
		MT_{MCI}&=ME_{\hat F_{12S},\hat F_{12S}}[\tilde k_{1S}(x_{1S},x'_{1S})k_S(x_{S},x'_{S})\tilde k_{2S}(x_{2S},x'_{2S})]\\
		&=M\sum^\infty_{i,j,q=1}\lambda_{1,i}\lambda_{2,j}\lambda_{S,q}E^2_{\hat F_{12S}}[\tilde\phi_{1,r}(x_{1S})\phi_{S,r}(x_S)\tilde\phi_{2,r}(x_{2S})]\\
		&=M\sum^\infty_{i,j,q=1}\lambda_{1,i}\lambda_{2,j}\lambda_{S,q}(\alpha^*E_{\hat U_{12S}}[\tilde\phi_{1,r}(x_{1S})\phi_{S,r}(x_S)\tilde\phi_{2,r}(x_{2S})]-\alpha^*E_{ U_{12S}}[\cdots]\\
		&+(1-\alpha^*)E_{\hat U'_{12S}}[\tilde\phi_{1,r}(x_{1S})\phi_{S,r}(x_S)\tilde\phi_{2,r}(x_{2S})]-(1-\alpha^*)E_{U'_{12S}}[\cdots])^2\\
		&=\sum^\infty_{i,j,q=1}\lambda_{1,i}\lambda_{2,j}\lambda_{S,q}\big(\sqrt{\frac{M}{n}}\sqrt{n}\alpha^*(E_{\hat U_{12S}}[\tilde\phi_{1,r}(x_{1S})\phi_{S,r}(x_S)\tilde\phi_{2,r}(x_{2S})]-E_{ U_{12S}}[\cdots])\\
		&+\sqrt{\frac{M}{n'}}\sqrt{n'}(1-\alpha^*)(E_{\hat U'_{12S}}[\tilde\phi_{1,r}(x_{1S})\phi_{S,r}(x_S)\tilde\phi_{2,r}(x_{2S})]-E_{U'_{12S}}[\cdots])\big)^2. 
	\end{align*}

	Following a similar procedure to the proof of Theorem \ref{thm:tci_asyd}, we can show the distributional convergence:
	\begin{align*}
		MT_{CI}\xrightarrow{d}\sum^\infty_{i,j,q=1}\lambda_{1,i}\lambda_{2,j}\lambda_{S,q}\xi^2_{ijq}, 
	\end{align*}
	where $(\dots,\xi_{ijq},\dots)$ follows a multivariate normal distribution of mean $\mathbf{0}$ and covariances 
	\begin{align*}
		Cov[\xi_{ijq},\xi_{i'j'q'}]&=\nu{\alpha^*}^2Cov_{U_{12S}}[\tilde\phi_{1,i}(x_{1S})\phi_{S,j}(x_S)\tilde\phi_{2,q}(x_{2S}),\tilde\phi_{1,i'}(x_{1S})\phi_{S,j'}(x_S)\tilde\phi_{2,q'}(x_{2S})]\\
		&+\nu'(1-\alpha^*)^2Cov_{U'_{12S}}[\tilde\phi_{1,i}(x_{1S})\phi_{S,j}(x_S)\tilde\phi_{2,q}(x_{2S}),\tilde\phi_{1,i'}(x_{1S})\phi_{S,j'}(x_S)\tilde\phi_{2,q'}(x_{2S})]. 
	\end{align*}
	
	Finally, we consider the behavior of $MT_{MCI}$ under $H_1$. In this case, according to Theorem 3 in \citet{fukumizu07ci},  $T_{MCI}$ converges to a positive value. Therefore, $MT_{MCI}\xrightarrow{p}\infty$, as $M\rightarrow\infty$.
\end{proof}

\begin{proof}[Proof of Theorem \ref{thm:tmci_mv_w/alpha}]
	The statistic $T_{MCI}$ can be written as 
	
	$$T_{MCI}=\frac{{\alpha^*}^2}{n^2}\sum^n_{i_1,i_2=1}\tilde k_{12S}(x^{(i_1)},x^{(i_2)})+\frac{(1-\alpha^*)^2}{n'^2}\sum^{n'}_{q_1,q_2=1}\tilde k_{12S}(x'^{(q_1)},x'^{(q_2)})+2\frac{\alpha^*(1-\alpha^*)}{nn'}\sum^n_{i_1=1}\sum^{n'}_{q_1=1}\tilde k_{12S}(x^{(i_1)},x'^{(q_1)}), $$
	where $\tilde k_{12S}$ is a kernel associated with a feature map $\tilde\varphi_{12S}(x):= (\varphi_1(x_1) - \mu_{x_1 \mid x_S}(x_S)) \otimes \varphi_S(x_S) \otimes (\varphi_2(x_2) - \mu_{x_2 \mid x_S}(x_S))$.

	This expression corresponds to a two-sample V-statistic, which can be rewritten as 
	\begin{align*}
		T_{MCI}=\frac{1}{n^2n'^2}\sum^n_{i_1,i_2=1}\sum^{n'}_{q_1,q_2=1}f_{i_1,i_2,q_1,q_2},
	\end{align*}
	
	where we define a symmetric function
	$$f_{i_1,i_2,q_1,q_2}:=\frac{1}{4}\sum^{(i_1,i_2)}_{(j_1,j_2)}\sum^{(q_1,q_2)}_{(r_1,r_2)}\langle\alpha^*\tilde\varphi_{12S}(x^{(j_1)})+(1-\alpha^*)\tilde\varphi_{12S}(x'^{(r_1)}),\alpha^*\tilde\varphi_{12S}(x^{(j_2)})+(1-\alpha^*)\tilde\varphi_{12S}(x'^{(r_2)})\rangle.$$
	
	Then, similarly to the proof of Theorem \ref{thm:tci_mv_w/alpha}, we can derive the desired limits of $E[MT_{MCI}]$ and $V[MT_{MCI}]$.
\end{proof}

\section{WEAKLY-SUPERVISED KERNEL CI AND MCI TEST WITHOUT TRUE MIXTURE PROPORTIONS}\label{appsec:test_wo_prop}
\subsection{Proofs for the tests without true mixture proportions}\label{appsubsec:proof_test_wo_prop}
\begin{proof}[Proof of Lemma \ref{lem:asymp_exp}]
	By the Taylor expansion of $T_{\hat\alpha}$ around $\alpha^*$, we derive
	\begin{align}\label{eq:taylorexp}
		MT_{\hat\alpha}&=MT_{\alpha^*}+M(\hat\alpha-\alpha^*)T'_{\alpha^*}+M\frac{1}{2}(\hat\alpha-\alpha^*)^2T''_{\alpha^*}+o_p(1)
	\end{align}
	
	The remainder term is $o_p(1)$ because $\sqrt{M}(\hat\alpha-\alpha^*)$ converges in distribution to a normal random variable by Theorem \ref{thm:cimpe_an} and \ref{thm:mcimpe_an}, which ensures that higher-order terms in the expansion vanish in probability.
\end{proof}

\begin{proof}[Proof of Theorem \ref{thm:test_consistency_w/oalpha}]
	Under Assumption \ref{asmp:testconswomp} and $H_1$, $T_{\hat\alpha}$ converges to the population test statistics of $T_{\alpha_1}$ as $M\rightarrow\infty$. Since $F_{12}^{\alpha_1}$ (resp. $F_{12S}^{\alpha_1}$) does not satisfy Conditional Independence for the CI test (resp. Multivariate CI for the MCI test), with assumptions in Theorem \ref{thm:tci_asyd} (resp. Theorem \ref{thm:tmci_asyd}), we can show that the population $T_{\alpha_1}$ is a positive constant. Then, $MT_{\hat\alpha}\xrightarrow{p}\infty$. 
\end{proof}

\subsection{Mean and variance estimation for the tests without true mixture proportions}\label{appsubsec:mean_var_wo_prop}
In this subsection, we explain how to estimate the mean and variance for the tests without true mixture proportions. Our approach utilizes the result of Lemma \ref{lem:asymp_exp}. We begin by analyzing the asymptotic behavior of each term in Equation \ref{eq:taylorexp}.

\subsubsection{Asymptotic behaviors of each term in the Taylor expansion of $MT_{\hat\alpha}$}
By Theorem \ref{thm:tci_asyd} and \ref{thm:tmci_asyd}, $MT_{\alpha^*}$	converges in distribution to a sum of squared normal random variables. By Theorem \ref{thm:cimpe_an} and \ref{thm:mcimpe_an}, $\sqrt{M}(\hat\alpha-\alpha^*)$ converges to a normal distribution. The term $T''_{\alpha^*}$ converges to a constant in probability. The remaining term $\sqrt{M}T'_{\alpha^*}$ converges to a normal distribution, which is derived as follows.	

$T'_{\alpha^*}$ is a V-statistic. For the CI test, using the same notations as in the proof of Theorem \ref{thm:tci_mv_w/alpha}, it can be expressed as
$$T'_{\alpha^*}=\frac{1}{n^6n'^6}\sum^{n}_{i_1,...,i_6=1}\sum^{n'}_{q_1,...,q_6=1}h'_{i_1,...,i_6,q_1,...,q_6}$$
where
\begin{align*}
	&h'_{i_1, \ldots, i_6, q_1, \ldots, q_6}:=\frac{1}{6!6!} \sum_{\left(j_1, \ldots, j_6\right)}^{\left(i_1, \ldots, i_6\right)} \sum_{\left(r_1, \ldots, r_6\right)}^{\left(q_1, \ldots, q_6\right)}
	2\Bigl<\Bigl(\tilde{\varphi}_{12}(x^{(j_1)})- \tilde{\varphi}_{12}(x^{\prime(r_1)})-\left(\alpha^* \tilde{\varphi}_1(x_1^{(j_2)})+(1-\alpha^*) \tilde{\varphi}_1(x_2^{\prime(r_2)})\right)\\ &\otimes\left( \tilde{\varphi}_2(x_1^{\left(j_3\right)})- \tilde{\varphi}_2(x_2^{\prime\left(r_3\right)})\right)
	-\left(\tilde{\varphi}_1(x_1^{\left(j_2\right)})- \tilde{\varphi}_1(x_2^{\prime\left(r_2\right)})\right) \otimes\left(\alpha^* \tilde{\varphi}_2(x_1^{\left(j_3\right)})+\left(1-\alpha^*\right) \tilde{\varphi}_2(x_2^{\prime\left(r_3\right)})\right)\Bigr), \varphi_{j_4, \ldots, j_6, r_4, \ldots, r_6}\Bigr>.
\end{align*}

For the MCI test, using the same notations as in the proof of Theorem \ref{thm:tmci_mv_w/alpha},
$$ 
T'_{\alpha^*}=\frac{1}{n^2 n^{\prime 2}} \sum_{i_1, i_2=1}^n \sum_{q_1, q_2=1}^{n^{\prime}}f'_{i_1, i_2, q_1, q_2},
$$
where 
$$
f'_{i_1, i_2, q_1, q_2}:=\frac{1}{4} \sum_{\left(j_1, j_2\right)}^{\left(i_1, i_2\right)} \sum_{\left(r_1, r_2\right)}^{\left(q_1, q_2\right)} 2\left\langle\alpha^* \tilde{\varphi}_{12 S}(x^{(j_1)})+\left(1-\alpha^*\right) \tilde{\varphi}_{12 S}(x'^{(r_1)}), \tilde{\varphi}_{12 S}(x^{(j_2)})- \tilde{\varphi}_{12 S}(x'^{(r_2)})\right\rangle .
$$

In both the CI and MCI cases, $T'_{\alpha^*}$ is non-degenerate, since in general, $$E_{i_2,...,i_6,q_1,...,q_6}[h'_{i_1, \ldots, i_6, q_1, \ldots, q_6}]\neq0, E_{i_1,...,i_6,q_2,...,q_6}[h'_{i_1, \ldots, i_6, q_1, \ldots, q_6}]\neq0$$
and 
$$E_{i_2,q_1,q_2}[f'_{i_1, i_2, q_1, q_2}]\neq0, E_{i_1,i_2,q_1}[f'_{i_1, i_2, q_1,  q_2}]\neq0.$$

Since non-degenerate V-statistics has $\sqrt{M}$-asymptotic normality \citep{huang23weighted, serfling81app}, $\sqrt{M}T'_{\alpha^*}$ converges to a normal distribution. 

\subsubsection{Mean and Variance estimation of $T_{\hat\alpha}$ for the CI and MCI tests}
We consider the mean and variance estimation for the CI test without true mixture proportions in this subsection. A similar derivation process can be applied to the MCI test. As with $T_\alpha$, denote $\check T_{CI}$ as $\check T_\alpha$, specifying $\check
T_{CI}$ is a function of $\alpha$. Let $\check T'_\alpha$ and $\check T''_\alpha$ be the first and second order derivatives of $\check T_\alpha$ at $\alpha$. Note that $\check{T}_{\alpha^*}$, $\check{T}'_{\alpha^*}$ and $\check{T}''_{\alpha^*}$  are V-statistics and we denote their corresponging U-statistics by $\check T_{U,\alpha^*}$, $\check T'_{U,\alpha^*}$ and $\check T''_{U,\alpha^*}$. We assume $T''_{U,\alpha^*}$ converges to a constant $c_0$ in probability. To estimate the expectation and variance of the right hand side of Equation \ref{eq:taylorexp}, we simplify the equation by considering the asymptotic behavior of each term. 

First, we can show $M(T_{\alpha^*}-\check{T}_{\alpha^*})\xrightarrow{p}0$, $\sqrt{M}(T'_{\alpha^*}-\check{T}'_{\alpha^*})\xrightarrow{p}0$ and $T''_{\alpha^*}-\check{T}''_{\alpha^*}\xrightarrow{p}0$, following a similar analysis to that used to derive Theorem \ref{thm:tci_asyd}.  Given the asymptotic equivalence of U-statistics and V-statistics (Lemma S5. in the supplement of \citet{huang23weighted}), we obtain the following convergence results as $M\rightarrow\infty$,
\begin{align*}
	M(T_{\alpha^*}- \check T_{U,\alpha^*})&\xrightarrow{p}c_1,\\ 
	\sqrt{M}(T'_{\alpha^*}-\check T'_{U,\alpha^*})&\xrightarrow{p}0,\\ 
	T''_{\alpha^*}-\check T''_{U,\alpha^*}&\xrightarrow{p}0,
\end{align*}
where $c_1$ is a constant.

In addition, following the procedure in the proof of Theorem \ref{thm:cimpe_an},
$$\sqrt{M}\left((\hat\alpha-\alpha^*)-S_{\alpha^*}\right)\xrightarrow{p}0,$$

where  $S_{\alpha^*}:=-\frac{1}{d_0}(\alpha^*E_{\hat U_{12}}[\tilde g_{12}]+(1-\alpha^*)E_{\hat U'_{12}}[\tilde g_{12}])=\frac{1}{nn'}\sum^{n}_{i=1}\sum^{n'}_{q=1} l_{i,q}$ and $l_{i,q}:=-\frac{1}{d_0}(\alpha^*\tilde g_{12}(x^{(i)})+(1-\alpha^*)\tilde g_{12}(x'^{(q)}))$.

Combining these results, the Taylor expansion can be approximated by
\begin{align}\label{eq:approxbyustat}
	M\{T_{\alpha^*}+(\hat\alpha-\alpha^*)T'_{\alpha^*}+\frac{1}{2}(\hat\alpha-\alpha^*)^2T''_{\alpha^*}\}\simeq M\{\check T_{U,\alpha^*}+S_{\alpha^*}\check T'_{U,\alpha^*}+\frac{c_0}{2}S_{\alpha^*}^2\}+c_1
\end{align}

Therefore, to perform the hypothesis test, we estimate the mean and variance of the right hand side of (\ref{eq:approxbyustat}). Under mild conditions such as uniform integrability, the asymptotic means and variances of both sides are actually equivalent. Using the same notation as in the proof of Theorem \ref{thm:tci_mv_w/alpha}, the U-statistics $\check T_{U,\alpha^*}$ and $\check T'_{U,\alpha^*}$ are expressed as 
\begin{align*}
	\check T_{U,\alpha^*}=&\frac{1}{(n)_2(n')_2}\sum_{i_1\neq i_2}\sum_{q_1\neq q_2}\check h_{i_1,i_2,q_1,q_2}\\
	\check T'_{U,\alpha^*}=&\frac{1}{(n)_2(n')_2}\sum_{i_1\neq i_2}\sum_{q_1\neq q_2}\check h'_{i_1,i_2,q_1,q_2}
\end{align*}
where we define $\check h'_{i_1,i_2,q_1,q_2}$ as
$$\check h'_{i_1,i_2,q_1,q_2}:=\frac{1}{2!2!}\sum^{(i_1,i_2)}_{(j_1,j_2)}\sum^{(q_1,q_2)}_{(r_1,r_2)}2\langle\alpha^*\tilde\varphi_{12}(x^{(j_1)})+(1-\alpha^*)\tilde\varphi_{12}(x'^{(r_1)}),\tilde\varphi_{12}(x^{(j_2)})-\tilde\varphi_{12}(x'^{(r_2)})\rangle.$$

\textbf{Mean Estimation:}
We next consider estimating the mean of (\ref{eq:approxbyustat}), $M\{E[\check T_{U,\alpha^*}]+E[S_{\alpha^*}\check T'_{U,\alpha^*}]+\frac{c_0}{2}E[S_{\alpha^*}^2]\}+c_1$. We can derive the asymptotic mean of each term as follows, as $M\rightarrow\infty$,
\begin{align*}
	ME[\check T_{U,\alpha^*}]+c_1&=c_1=\nu{\alpha^*}^2E_{i_1,i_2}[\tilde k_{12}(x^{(i_1)},x^{(i_1)})-\tilde 	k_{12}(x^{(i_1)},x^{(i_2)})]\\
	&+\nu'(1-\alpha^*)^2E_{q_1,q_2}[\tilde k_{12}(x'^{(q_1)},x'^{(q_1)})-\tilde k_{12}(x'^{(q_1)},x'^{(q_2)})],\\	 
	ME[S_{\alpha^*}\check T'_{U,\alpha^*}]&\rightarrow-\frac{2}{d_0}\left(\nu\alpha^*E_{i_1}\left[\tilde g_{12}(x^{(i_1)})E_{i_2,q_1,q_2}[\check h'_{i_1,i_2,q_1,q_2}]\right]+ \nu'(1-\alpha^*)E_{q_1}\left[\tilde g_{12}(x'^{(q_1)})E_{i_1,i_2,q_2}[\check h'_{i_1,i_2,q_1,q_2}]\right]\right),\\
	ME[S^2_{\alpha^*}]&\rightarrow\frac{1}{d_0^2}\left(\nu{\alpha^*}^2V_i[\tilde g_{12}(x^{(i)})]+\nu'(1-\alpha^*)^2V_q[\tilde g_{12}(x'^{(q)})]\right).
\end{align*}

Here, the limit of $ME[\check T_{U,\alpha^*}]$ follows from Theorem \ref{thm:tci_mv_w/alpha}. The limit of $ME[S_{\alpha^*}\check T'_{U,\alpha^*}]$ is derived by considering only dominant terms of the multiplication $S_{\alpha^*}\check T'_{U,\alpha^*}$ and the fact that $E_{i_1,i_2,q_1,q_2}[h'_{i_1,i_2,q_1,q_2}]=0$. The limit of $ME[S^2_{\alpha^*}]$ is equivalent to the asymptotic variance of $\hat\alpha_{CI}$.

\textbf{Variance Estimation:} Next, we consider estimating the variance of (\ref{eq:approxbyustat}), which is written as 
$$M^2\{V[\check T_{U,\alpha^*}]+V[S_{\alpha^*}\check T'_{U,\alpha^*}]+\frac{c^2_0}{4}V[S_{\alpha^*}^2]+2(Cov[\check T_{U,\alpha^*},S_{\alpha^*}\check T'_{U,\alpha^*}]+\frac{c_0}{2}Cov[\check T_{U,\alpha^*},S_{\alpha^*}^2]+\frac{c_0}{2}Cov[ S_{\alpha^*}\check T'_{U,\alpha^*},S^2_{\alpha^*}])\}.$$

Similarly to the asymptotic mean calculation, we derive the limit of each term by considering only dominant terms whose expectations are nonzero. As $M\rightarrow\infty$, we have the following convergence
{\allowdisplaybreaks\begin{align*}
	M^2V\left[\check T_{U,\alpha^*}\right]&\rightarrow 2\nu^2E_{i_1,i_2}\left[E_{q_1,q_2}^2\left[\check h_{i_1,i_2,q_1,q_2}\right]\right]+2\nu'^2E_{q_1,q_2}\left[E_{i_1,i_2}^2\left[\check h_{i_1,i_2,q_1,q_2}\right]\right]\\
	&+16\nu\nu'E_{i_1,q_1}\left[E_{i_2,q_2}^2\left[\check h_{i_1,i_2,q_1,q_2}\right]\right]
	\left(=\lim_{M\rightarrow\infty} V\left[MT_{CI}\right]\right),\\
	M^2V\left[S_{\alpha^*}\check T'_{U,\alpha^*}\right]&=M^2E\left[\left(S_{\alpha^*}\check T'_{U,\alpha^*}\right)^2\right]-M^2\left(E\left[S_{\alpha^*}\check T'_{U,\alpha^*}\right]\right)^2\\
	&\rightarrow 4\nu^2E_{i_1}\left[E^2_{q_1}\left[l_{i_1,q_1}\right]\right]E_{i_1}\left[E^2_{i_2,q_1,q_2}\left[\check h'_{i_1,i_2,q_1,q_2}\right]\right]+4\nu\nu'E_{i_1}\left[E^2_{q_1}\left[l_{i_1,q_1}\right]\right]E_{q_1}\left[E^2_{i_1,i_2,q_2}\left[\check h'_{i_1,i_2,q_1,q_2}\right]\right]\\
	&+4\nu\nu'E_{q_1}\left[E^2_{i_1}\left[l_{i_1,q_1}\right]\right]E_{i_1}\left[E^2_{i_2,q_1,q_2}\left[\check h'_{i_1,i_2,q_1,q_2}\right]\right]+4\nu'^2E_{q_1}\left[E^2_i\left[l_{i_1,q_1}\right]\right]E_{q_1}\left[E^2_{i_1,i_2,q_2}\left[\check h'_{i_1,i_2,q_1,q_2}\right]\right]\\
	&+8\nu^2E^2_{i_1}\left[E_{q_1}\left[l_{i_1,q_1}\right]E_{i_2,q_1,q_2}\left[\check h'_{i_1,i_2,q_1,q_2}\right]\right]\\
	&+16\nu\nu'E_{i_1}\left[E_{q_1}\left[l_{i_1,q_1}\right]E_{i_2,q_1,q_2}\left[\check h'_{i_1,i_2,q_1,q_2}\right]\right]E_{q_1}\left[E_{i_1}\left[l_{i_1,q_1}\right]E_{i_1,i_2,q_2}\left[\check h'_{i_1,i_2,q_1,q_2}\right]\right]\\
	&+8\nu'^2E^2_{q_1}\left[E_i\left[l_{i_1,q_1}\right]E_{i_1,i_2,q_2}\left[\check h'_{i_1,i_2,q_1,q_2}\right]\right]-\lim_{M\rightarrow\infty}M^2E^2\left[S_{\alpha^*}\check T'_{U,\alpha^*}\right]\\
	&= 4\left(\nu E_{i_1}\left[E^2_{q_1}\left[l_{i_1,q_1}\right]\right]+\nu'E_{q_1}\left[E^2_{i_1}[l_{i_1,q_1}]\right]\right)\\
	&\left(\nu E_{i_1}[E^2_{i_2,q_1,q_2}[\check h'_{i_1,i_2,q_1,q_2}]]+\nu'E_{q_1}[E^2_{i_1,i_2,q_2}[\check h'_{i_1,i_2,q_1,q_2}]]\right)\\
	&+8\left(\nu E_{i_1}\left[E_{q_1}\left[l_{i_1,q_1}\right]E_{i_2,q_1,q_2}[\check h'_{i_1,i_2,q_1,q_2}]\right]+\nu'E_{q_1}\left[E_i\left[l_{i_1,q_1}\right]E_{i_1,i_2,q_2}[\check h'_{i_1,i_2,q_1,q_2}]\right]\right)^2\\
	&-\lim_{M\rightarrow\infty}M^2E^2\left[S_{\alpha^*}\check T'_{U,\alpha^*}\right],\\
	M^2V\left[S_{\alpha^*}^2\right]&=M^2E\left[S_{\alpha^*}^4\right]-M^2E^2\left[S_{\alpha^*}^2\right]\\
	&\rightarrow 3\nu^2E^2_{i_1}\left[E^2_{q_1}\left[l_{i_1,q_1}\right]\right]+6\nu\nu'E_{i_1}\left[E^2_{q_1}\left[l_{i_1,q_1}\right]\right]E_{q_1}\left[E^2_{i_1}\left[l_{i_1,q_1}\right]\right]+3\nu'^2E^2_{q_1}\left[E^2_{i_1}\left[l_{i_1,q_1}\right]\right]\\&-\lim_{M\rightarrow\infty}M^2E^2\left[S_{\alpha^*}^2\right]\\
	&=  3\left(\nu E_{i_1}\left[E^2_{q_1}\left[l_{i_1,q_1}\right]\right]+\nu'E_{q_1}\left[E^2_{i_1}\left[l_{i_1,q_1}\right]\right]\right)^2-\lim_{M\rightarrow\infty}M^2E^2\left[S_{\alpha^*}^2\right],
\end{align*}
\begin{align*}
	M^2Cov[\check T_{U,\alpha^*},S_{\alpha^*}\check T'_{U,\alpha^*}]&=M^2E[\check T_{U,\alpha^*}S_{\alpha^*}\check T'_{U,\alpha^*}]-M^2E[\check T_{U,\alpha^*}]E[S_{\alpha^*}\check T'_{U,\alpha^*}]\\
	&\rightarrow 4\nu^2 E_{i_1,i_2}\left[E_{q_1,q_2}[\check h_{i_1,i_2,q_1,q_2}]E_{i_2,q_1,q_2}[\check h'_{i_1,i_2,q_1,q_2}]E_{q_1}[l_{i_2,q_1}]\right]\\
	&+8\nu\nu'E_{i_1,q_1}\left[E_{i_2,q_2}[\check h_{i_1,i_2,q_1,q_2}]E_{i_2,q_1,q_2}[\check h'_{i_1,i_2,q_1,q_2}]E_{i_1}[l_{i_1,q_1}]\right]\\
	&+8\nu\nu'E_{i_1,q_1}\left[E_{i_2,q_2}[\check h_{i_1,i_2,q_1,q_2}]E_{i_1,i_2,q_2}[\check h'_{i_1,i_2,q_1,q_2}]E_{q_1}[l_{i_1,q_1}]\right]\\
	&+4\nu'^2E_{q_1,q_2}\left[E_{i_1,i_2}[\check h_{i_1,i_2,q_1,q_2}]E_{i_1,i_2,q_2}[\check h'_{i_1,i_2,q_1,q_2}]E_{i_1}[l_{i_1,q_2}]\right],\\
	M^2Cov[\check T_{U,\alpha^*}, S^2_{\alpha^*}]&=M^2E[\check T_{U,\alpha^*}S^2_{\alpha^*}]-M^2E[\check T_{U,\alpha^*}]E[S^2_{\alpha^*}]\\
	&\rightarrow 2\nu^2E_{i_1,i_2}\left[E_{q_1,q_2}[\check h_{i_1,i_2,q_1,q_2}]E_{q_1}[l_{i_1,q_1}]E_{q_1}[l_{i_2,q_1}]\right]\\&
	+8\nu\nu'E_{i_1,q_1}\left[E_{i_2,q_2}[\check h_{i_1,i_2,q_1,q_2}]E_{q_1}[l_{i_1,q_1}]E_{i_1}[l_{i_1,q_1}]\right]\\
	&+2\nu'^2E_{q_1,q_2}\left[E_{i_1,i_2}[\check h_{i_1,i_2,q_1,q_2}]E_{i_1}[l_{i_1,q_1}]E_{i_1}[l_{i_1,q_2}]\right],\\
	M^2Cov[S_{\alpha^*}\check T'_{U,\alpha^*},S_{\alpha^*}^2]&=M^2E[S_{\alpha^*}^3\check T'_{U,\alpha^*}]-M^2E[S_{\alpha^*}\check T'_{U,\alpha^*}]E[S_{\alpha^*}^2]\\
	&\rightarrow 6\nu^2 E_{i_1}\left[E_{i_2,q_1,q_2}[\check h'_{i_1,i_2,q_1,q_2}]E_{q_1}[l_{i_1,q_1}]\right]E_{i_1}\left[E^2_{q_1}[l_{i_1,q_1}]\right]\\
	&+6\nu\nu' E_{i_1}\left[E_{i_2,q_1,q_2}[\check h'_{i_1,i_2,q_1,q_2}]E_{q_1}[l_{i_1,q_1}]\right]E_{q_1}\left[E^2_{i_1}[l_{i_1,q_1}]\right]\\	
	&+6\nu\nu' E_{q_1}\left[E_{i_1,i_2,q_2}[\check h'_{i_1,i_2,q_1,q_2}]E_{i_1}[l_{i_1,q_1}]\right]E_{i_1}\left[E^2_{q_1}[l_{i_1,q_1}]\right]\\
	&+6\nu'^2 E_{q_1}\left[E_{i_1,i_2,q_2}[\check h'_{i_1,i_2,q_1,q_2}]E_{i_1}[l_{i_1,q_1}]\right]E_{q_1}\left[E^2_{i_1}[l_{i_1,q_1}]\right]-\lim_{M\rightarrow\infty}M^2E[S_{\alpha^*}\check T'_{U,\alpha^*}]E[S_{\alpha^*}^2]\\
	&=6\left(\nu E_{i_1}\left[E_{i_2,q_1,q_2}[\check h'_{i_1,i_2,q_1,q_2}]E_{q_1}[l_{i_1,q_1}]\right]+\nu'E_{q_1}\left[E_{i_1,i_2,q_2}[\check h'_{i_1,i_2,q_1,q_2}]E_{i_1}[l_{i_1,q_1}]\right]\right)\\
	&\left(\nu E_{i_1}\left[E^2_{q_1}[l_{i_1,q_1}]\right]+\nu'E_{q_1}\left[E^2_{i_1}[l_{i_1,q_1}]\right]\right)-\lim_{M\rightarrow\infty}M^2E[S_{\alpha^*}\check T'_{U,\alpha^*}]E[S_{\alpha^*}^2].
\end{align*}
}

Furthermore, the expactations of $\check h_{i_1,i_2,q_1,q_2}$ and $\check h'_{i_1,i_2,q_1,q_2}$ are written as
\begin{align*}
	E_{i_1,i_2}[\check h_{i_1,i_2,q_1,q_2}]&=\left\langle\alpha^*E_{i_1}[\tilde\varphi_{12}(x^{(i_1)})]+(1-\alpha^*)\tilde\varphi_{12}(x'^{(q_1)}), \alpha^*E_{i_1}[\tilde\varphi_{12}(x^{(i_1)})]+(1-\alpha^*)\tilde\varphi_{12}(x'^{(q_2)})\right\rangle,\\
	E_{q_1,q_2}[\check h_{i_1,i_2,q_1,q_2}]&=\left\langle\alpha^*\tilde\varphi_{12}(x^{(i_1)})+(1-\alpha^*)E_{q_1}[\tilde\varphi_{12}(x'^{(q_1)})], \alpha^*\tilde\varphi_{12}(x^{(i_2)})+(1-\alpha^*)E_{q_1}[\tilde\varphi_{12}(x'^{(q_1)})]\right\rangle,\\
	E_{i_2,q_2}[\check h_{i_1,i_2,q_1,q_2}]&=\frac{1}{2} \left\langle\alpha^*\tilde\varphi_{12}(x^{(i_1)})+(1-\alpha^*)E_{q_2}[\tilde\varphi_{12}(x'^{(q_2)})], \alpha^*E_{i_2}[\tilde\varphi_{12}(x^{(i_2)})]+(1-\alpha^*)\tilde\varphi_{12}(x'^{(q_1)})\right\rangle,\\
	E_{i_1,i_2,q_2}[\check h'_{i_1,i_2,q_1,q_2}]&=\left\langle\alpha^*E_{i_1}[\tilde\varphi_{12}(x^{(i_1)})]+(1-\alpha^*)\tilde\varphi_{12}(x'^{(q_1)}), E_{i_1}[\tilde\varphi_{12}(x^{(i_1)})]-E_{q_1}[\tilde\varphi_{12}(x'^{(q_1)})]\right\rangle\\
	&=-\left\langle E_{i_1}[\tilde\varphi_{12}(x^{(i_1)})],E_{q_1}[\tilde\varphi_{12}(x'^{(q_1)})]\right\rangle+\left\langle E_{i_1}[\tilde\varphi_{12}(x^{(i_1)})],\tilde\varphi_{12}(x'^{(q_1)})\right\rangle,\\
	E_{i_2,q_1,q_2}[\check h'_{i_1,i_2,q_1,q_2}]&=\left\langle\alpha^*\tilde\varphi_{12}(x^{(i_1)})+(1-\alpha^*)E_{q_1}[\tilde\varphi_{12}(x'^{(q_1)})], E_{i_1}[\tilde\varphi_{12}(x^{(i_1)})]-E_{q_1}[\tilde\varphi_{12}(x'^{(q_1)})]\right\rangle\\
	&=-\left\langle \tilde\varphi_{12}(x^{(i_1)}),E_{q_1}[\tilde\varphi_{12}(x'^{(q_1)})]\right\rangle+\left\langle E_{i_1}[\tilde\varphi_{12}(x^{(i_1)})],E_{q_1}[\tilde\varphi_{12}(x'^{(q_1)})]\right\rangle.
\end{align*}

With these results, we have derived expressions for the asymptotic mean and variance of  $MT_{\hat\alpha}$. In practice, each term in these expressions is estimated by replacing the population distributions $U_{12}$, $U'_{12}$ and $\alpha^*$ with their empirical counterparts $\hat U_{12}$, $\hat U'_{12}$ and $\hat\alpha$.

For the MCI test, the asymptotic mean and variance of $T_{\hat\alpha}$ can be estimated similarly. This is done by replacing the terms from the CI test, namely $\tilde g_{12}$, $\tilde{k}_{12}$, $\check h_{i_1,i_2,q_1,q_2}$, $\check h'_{i_1,i_2,q_1,q_2}$,  $l_{i_1,q_1}$, with $\tilde g_{12S}$, $\tilde{k}_{12S}$, $f_{i_1,i_2,q_1,q_2}$, $f'_{i_1,i_2,q_1,q_2}$ and $l_{i_1,q_1}^{MCI}$, respectively. Here, we define
$$l_{i_1,q_1}^{MCI}:=-\frac{1}{d_0}(\alpha^*\tilde g_{12S}(x^{(i_1)})+(1-\alpha^*)\tilde g_{12S}(x'^{(q_1)})).$$

\section{EXPERIMENTS}\label{appsec:exp}
\subsection{Practical computation of test statistic}\label{appsubsec:comp_test_stat}
In this section, we explain how to calculate the test statistics in practice.
Let $K_\tau\in\mathbb{R}^{M\times M}$ be the Gram matrix of $x_\tau$ with a kernel $k_\tau$, where the entries are given by $(K_{\tau})_{ij}=k_\tau(\mathbf{v}_{x_\tau,i},\mathbf{v}_{x_\tau,j})$ . Define $D_\alpha\in\mathbb{R}^{M\times M}$ as a diagonal matrix where $(D_{\alpha})_{ii}=\alpha/n$ for $i\leq n$ and $(D_{\alpha})_{ii}=(1-\alpha)/n'$ for $i> n$. Let $\mathbf{1}$ be an $M\times 1$ vector of ones and define the centering matrix $H:=(I-\mathbf{1}\mathbf{1}^TD_{\alpha^*})$.  Then, $T_{CI}$ can be computed in matrix form as
$$T_{CI}=\operatorname{tr}(HK_1H^TD_{\alpha^*}HK_2H^TD_{\alpha^*}). $$

Using centralized kernels $\tilde k_{1S}(x_{1S},x'_{1S})=\langle 
\varphi_1(x_1)-\mu_{X_1 \mid X_S}(x_S)
,\varphi_1(x'_1)-\mu_{X_1 \mid X_S}(x'_S)
\rangle$ and $\tilde k_{2S}(x_{2S},x'_{2S})=\langle 
\varphi_2(x_2)-\mu_{X_2 \mid X_S}(x_S)
,\varphi_2(x'_2)-\mu_{X_2 \mid X_S}(x'_S)
\rangle$,we can compute $T_{MCI}$ as 

$$T_{MCI}=\operatorname{tr}((\tilde K_{1S} \odot K_S) D_{\alpha^*} \tilde K_{2S} D_{\alpha^*})$$

where $\odot$ denotes the Hadamard product. $\tilde K_{1S}\in\mathbb{R}^{M\times M}$ and $\tilde K_{2S}\in\mathbb{R}^{M\times M}$ are the Gram matrices associated with $\tilde k_{1S}$ and $\tilde k_{2S}$, defined by $(\tilde K_{1S})_{ij}=\tilde k_{1S}(\mathbf{v}_{x_{1S},i},\mathbf{v}_{x_{1S},j})$ and $(\tilde K_{2S})_{ij}=\tilde k_{2S}(\mathbf{v}_{x_{2S},i},\mathbf{v}_{x_{2S},j})$. 

For $T_{MCI}$, we need to estimate the centralized kernels $\tilde k_{1S}$ and $\tilde k_{2S}$. To do this, we consider the eigenvalue decomposition of the kernel matrix, $K_\tau=V^T_\tau \Lambda_\tau V_\tau$, which provides an empirical kernel map $\hat\varphi_\tau=[\hat\varphi_{\tau,1}(\mathbf{v}_{x_\tau}),...,\hat\varphi_{\tau,n+n'}(\mathbf{v}_{x_\tau})]=V_\tau \Lambda_\tau^{1/2}$. In the original KCI test \citep{zhang11kci}, each feature map $\hat\varphi_{\tau,i}(\mathbf{v}_{x_\tau})$ is centralized as $\hat\varphi_{\tau,i}(\mathbf{v}_{x_\tau})-E[\hat\varphi_{\tau,i}(\mathbf{v}_{x_\tau})|Z]$ by estimating the conditional expectation $E[\hat\varphi_{\tau,i}(\mathbf{v}_{x_\tau})|Z]$ using kernel ridge regression. 

In our setting, KRR is performed in a weakly-supervised manner, similarly to the MCI MPE in Section \ref{sec:mpemci}. It is optimized by the first-order condition of non-convex loss. To reduce the computational costs, we implement KRR only for the eigenvectors corresponding to the top $k$ eigenvalues and omit other eigenvectors when reconstructing the centralized Gram matrix $\tilde K_{\tau}$. In our experiments, we set $k=5$, which we found to be sufficient to approximate the original Gram matrix accurately.

\subsection{Experimental details}
\subsubsection{CI MPE with synthetic data}\label{appsubsec:expcimpe_syn}
For the UCI datasets, the positive and negative classes were assigned as shown in Table \ref{tab:pnuci}. We set the search range for the mixture proportion to be $I_{\alpha^+}=\mathbb{R}_+$.

\begin{table}[h!]\centering
	\caption{Positive and Negative classes used for UCI dataset}
	\label{tab:pnuci}
	\begin{tabular}{lcc}\\\toprule
		& Positive & Negative      \\\midrule
		Shuttle  & 1        & other classes \\
		Wine     & white    & red           \\
		Dry Bean & DERMASON & other classes\\\bottomrule
	\end{tabular}
\end{table}

\subsubsection{CI MPE with real-world data}\label{appsubsec:expcimpe_real}

We used two datasets from the UCI repository: the Breast Cancer Wisconsin and Dry Bean datasets. For each dataset, we chose positive and negative classes and implemented the experiments, switching the classes.  The procedure was as follows:

\begin{enumerate}
	\item  We first selected a candidate set of features $X_i$ that were discriminative, satisfying $\left|E\left[X_i \mid Y=1\right]-E\left[X_i \mid Y=-1\right]\right| / \sqrt{V\left[X_i \mid Y=1\right]}>0.5$, since a significant mean difference is essential for the efficient MPE.
	\item  We then applied the HSIC test to all pairs of features from this candidate set to identify those satisfying the CI condition, with a significance level 0.05.
	\item  For each detected CI feature pair, we ran our CI MPE method 10 times.
\end{enumerate}

For the MPE task, we set $n=n^{\prime}=2000$ and used a Positive-Unlabeled (PU) setting with class $\operatorname{priors}\left(\theta, \theta^{\prime}\right)=(1,0.5)$.

\subsubsection{MCI MPE with synthetic data}\label{appsubsec:expmcimpe_syn}
We used a regularization parameter $\lambda=5\times 10^{-4}$ and a Gaussian kernel with bandwidth $\sigma=3.5$ for all MCI MPE experiments. The search ranges were set to  $I_{\alpha_+}=[1.1,1.5]$ and $I_{\alpha_-}=[-0.7,0]$.

\subsubsection{MCI MPE with real-world data}\label{appsubsec:expmcimpe_real}
We used the Dry Bean dataset and set the positive and negative classes to SIRA and DERMASON, respectively. The procedure was as follows:

\begin{enumerate}
	\item We searched for feature triplets $(X_1,X_2,X_S)$ satisfying the MCI condition in the negative class by applying the KCI test \citep{zhang11kci} to all possible triplets with a significance level 0.05. In the search, we constructed a candidate set of features that satisfies $\left|E\left[X_i \mid Y=1\right]-E\left[X_i \mid Y=-1\right]\right| / \sqrt{V\left[X_i \mid Y=1\right]}>1$, similarly to \ref{appsubsec:expcimpe_real}. Then we only used features in the set as the candidates for $X_1$ and $X_2$. 
	\item For each detected triplet, we ran our MCI MPE method 5 times and evaluated the estimation error for $\theta'$.
\end{enumerate}

For the MPE task, we set $n=n^{\prime}=1000$ and used a Positive-Unlabeled (PU) setting with class $\operatorname{priors}\left(\theta, \theta^{\prime}\right)=(1,0.5)$. We used a Gaussian kernel with bandwidth $\sigma=1.0$ for KRR, set the regularization parameter to $\lambda=10^{-3}$ and the search range $I_{\alpha_-}=[-1.25,-0.5]$. In this experiment, 48 triplets were detected and the resulting MAE of $\hat\theta'$ over all runs was $0.0312 \pm 0.0304$.

\subsubsection{Hyperparameters for CI and MCI test}\label{appsubsec:exptest_hypara}
We used a Gaussian kernel for all CI and MCI test experiments,  both with and without mixture proportions. For CI test of both cases, kernel bandwidth is set as $\sigma=2.5$. For MCI test with mixture proportions, we set $\lambda=5\times10^{-4}$ and $\sigma=3.5$ for all kernels. For MCI test without mixture proportions, we set $\lambda=5\times10^{-6}$ and $\sigma=2.5$ for the test statistic, and $\lambda=1\times 10^{-2}$ and $\sigma=3$ for the MCI MPE to estimate $\alpha^*$. The search ranges for MPE are set as the same values in Section \ref{appsubsec:expcimpe_syn} and \ref{appsubsec:expmcimpe_syn}.

\subsection{Additional experiments}
\subsubsection{Bias calculation of CI and MCI MPE}\label{appsubsec:exp_bias}
We conducted an additional experiment to investigate the relationship between MPE error and the degree of CI violation (i.e., correlation). We used the same Gaussian data generation process as in Section \ref{subsec:exptest}, where the CI or MCI assumption is only satisfied when $\sigma_{12}=0$.  Each experiment was repeated 100 times with sample sizes $n=n'=2000$ and true class priors $\left(\theta, \theta^{\prime}\right)=(0.8,0.2)$.

The results are presented in Tables \ref{tab:bias_cimpe} and \ref{tab:bias_mcimpe}. As shown, the MPE error remains small even when the CI or MCI assumption is weakly violated.

\begin{table}[h!]\centering
	\caption{Mean absolute error of $(\hat{\theta}, \hat{\theta}^{\prime})$ with weakly-violated CI data}
	\label{tab:bias_cimpe}
	\begin{tabular}{ll}\\\toprule
		$\sigma_{12}^2$ & MAE of $(\hat{\theta}, \hat{\theta}^{\prime})$ \\
		\hline 0 & $(0.026 \pm 0.018,0.025 \pm 0.019)$ \\
		0.1 & $(0.074 \pm 0.034,0.030 \pm 0.021)$ \\
		0.2 & $(0.132 \pm 0.025,0.034 \pm 0.021)$\\\bottomrule
	\end{tabular}
\end{table}

\begin{table}[h!]\centering
	\caption{Mean absolute error of $(\hat{\theta}, \hat{\theta}^{\prime})$ with weakly-violated MCI data}\label{tab:bias_mcimpe}
	\begin{tabular}{ll}\\\toprule
		$\sigma_{12}^2$ & $\operatorname{MAE}$ of $(\hat{\theta}, \hat{\theta}^{\prime})$ \\
		\hline 0 & $(0.008 \pm 0.005,0.015 \pm 0.010)$ \\
		0.1 & $(0.017 \pm 0.016,0.014 \pm 0.012)$ \\
		0.2 & $(0.037 \pm 0.011,0.010 \pm 0.011)$\\\bottomrule
	\end{tabular}
\end{table}



\subsubsection{Investigation of low power in the MCI test without true mixture proportions}\label{appsubsec:exp_low_power}
To investigate the cause of the low test power, we compared our method to a test using the true null distribution (simulated with 1000 trials). We used the same setup as in Section \ref{subsec:exptest} ($H_0 : \sigma_{12} = 0$, $H_1 : \sigma_{12} = 0.2$, $n = n' = 1000$). Table \ref{tab:mci_test_power} summarizes the results.

\begin{table}[h!]
	\centering
	\caption{MCI test power: our null approximation vs. true null distribution}
	\label{tab:mci_test_power}
	\begin{tabular}{lc}
		\toprule
		Method & Test power \\
		\midrule
		Test with our null approximation & 0.199 \\
		Test with true null distribution & 0.293 \\
		\bottomrule
	\end{tabular}
\end{table}

As shown in Table \ref{tab:mci_test_power}, even with the true null distribution, the ideal test power remains low (0.293). This indicates that the poor separation between the true null and alternative distributions causes the low power in small samples.
	
\end{document}